%% file: main.tex
\documentclass{article}
\usepackage{graphicx} 
\usepackage[utf8]{inputenc}
\usepackage[margin=1in]{geometry}
\input{math_commands.tex}

\usepackage{hyperref}
\hypersetup{colorlinks,citecolor=blue,linkcolor=blue}
\usepackage{url}
\usepackage{amsthm}
\usepackage{amssymb,mathtools}
\usepackage[nameinlink,capitalise]{cleveref} 
\usepackage{graphicx}
\usepackage{subfigure}
\usepackage{multirow}
\usepackage{paralist}
\crefname{ineq}{Inequality}{Inequalities}
\usepackage{enumitem}

\usepackage{array} 
\usepackage{authblk}

\newtheorem{definition}{Definition}[section]
\newtheorem{theorem}{Theorem}[section]
\newtheorem{lemma}{Lemma}[section]
\newtheorem{example}{Example}[section]
\newtheorem{fact}{Fact}[section]
\newtheorem{corollary}{Corollary}[section]

\newtheorem{claim}{Claim}[section]

\newtheorem{remark}{Remark}[section]

\def\+#1{\mathcal{#1}}
\def\-#1{\mathbb{#1}}

\newcommand{\notshow}[1]{{}}
\newcommand{\AutoAdjust}[3]{{ \mathchoice{ \left #1 #2  \right #3}{#1 #2 #3}{#1 #2 #3}{#1 #2 #3} }}
\newcommand{\Xcomment}[1]{{}}

\newcommand{\InParentheses}[1]{\AutoAdjust{(}{#1}{)}}
\newcommand{\InBrackets}[1]{\AutoAdjust{[}{#1}{]}}

\newcommand{\InAbs}[1]{\AutoAdjust{|}{#1}{|}}
\renewcommand{\part}[2]{\frac{\partial #1}{\partial #2}}

\DeclareMathOperator{\QUERYCPLX}{QC}
\DeclareMathOperator{\QC}{QC}

\title{Learning Thresholds with Latent Values and Censored Feedback}
\author[1]{Jiahao Zhang}
\author[2]{Tao Lin}
\author[3]{Weiqiang Zheng}
\author[4]{Zhe Feng}
\author[4]{Yifeng Teng}
\author[1]{Xiaotie Deng}

\affil[1]{Peking University, \{jiahao.zhang, xiaotie\}@pku.edu.cn}
\affil[2]{Harvard University, tlin@g.harvard.edu}
\affil[3]{Yale University, weiqiang.zheng@yale.edu}
\affil[4]{Google Research, \{zhef, yifengt\}@google.com}

\begin{document}
\maketitle

\begin{abstract}
In this paper, we investigate a problem of \emph{actively} learning threshold in latent space, where the \emph{unknown} reward $g(\gamma, v)$ depends on the proposed threshold $\gamma$ and latent value $v$ and it can be \emph{only} achieved if the threshold is lower than or equal to the \emph{unknown} latent value. 
This problem has broad applications in practical scenarios, e.g., reserve price optimization in online auctions, online task assignments in crowdsourcing, setting recruiting bars in hiring, etc.
We first characterize the query complexity of learning a threshold with the expected reward at most $\eps$ smaller than the optimum and prove that the number of queries needed can be infinitely large even when $g(\gamma, v)$ is monotone with respect to both $\gamma$ and $v$. On the positive side, we provide a tight query complexity $\tilde{\Theta}(1/\eps^3)$ when $g$ is monotone and the CDF of value distribution is Lipschitz. Moreover, we show a tight $\tilde{\Theta}(1/\eps^3)$ query complexity can be achieved as long as $g$ satisfies right Lipschitzness, which provides a complete characterization for this problem. Finally, we extend this model to an online learning setting and demonstrate a tight $\Theta(T^{2/3})$ regret bound using continuous-arm bandit techniques and the aforementioned query complexity results.

\end{abstract}

\section{Introduction}
The thresholding strategy is widely used in practice, e.g., setting bars in the hiring process, setting reserve prices in online auctions, and setting requirements for online tasks in crowdsourcing, due to its intrinsic simplicity and transparency. In addition, in the practical scenarios mentioned above, the threshold can be only set in a \emph{latent} space, which makes the problem more challenging:
\begin{itemize}[leftmargin=*]
\item In hiring, the recruiter wants to set a recruiting bar that maximizes the quality of the hires without knowing the true qualifications of each candidate.
\item In online auctions, the seller wants to set a reserve price that maximizes their revenue. However, the seller does not know the true value of the item being auctioned.
\item In crowdsourcing, the taskmaster wants to set a requirement (e.g., the number of questions that need to be answered) for each task so that the tasks can be assigned to qualified workers who have a strong willingness to complete the task, but has no access to the willingness of each worker. The target of the taskmaster is to collect high-quality data as much as possible.
\end{itemize}

Although the latent value cannot be observed, it will affect the reward jointly with the threshold. In the examples above, the reward is the quality of the hire, revenue from the auction,  or the quality of the collected data. In some cases, the reward only depends on the latent value as long as the value is larger than the threshold, e.g., the quality of the hire may only depend on the intrinsic qualification and skill level of the candidate as long as she passes the recruiting bar, the revenue in second price auction only depends on the value since the bidders will report truthfully regardless of the reserve price. However, the threshold does play an important role in other settings, e.g., the winning bid in the first price auction relies on and increases with the reserve price in general~\cite{auction-theory-book} thus the revenue highly depends on both reserve price and latent value. Similarly, in crowdsourcing, the requirements of the tasks directly affect the quality of data collected by the taskmaster. In this work, we consider a general framework where the reward can depend on both threshold and latent value. 

Another difficulty in setting a proper threshold is to balance the per-entry reward and capability. For example, in hiring, setting a higher recruiting bar can guarantee the qualification of each candidate and thus the return to the company.  However, a too-high bar may reject all candidates and cannot fit the position needs. To balance this tradeoff, the threshold needs to be set appropriately.

\subsection{Our Model and Contribution}
\paragraph{An informal version of our model} In this work, we propose a general active learning abstraction for the aforementioned thresholding problem. We first provide some notations of the model considered in this paper to facilitate the presentation. Let $g(\gamma, v)$ be the reward function that maps the threshold $\gamma\in [0, 1]$ and latent value $v\in [0, 1]$ to the reward, where $g(\gamma, v)$ can be observed if and only if $\gamma \leq v$. We assume that the latent value follows an \emph{unknown} distribution and we can only observe the reward (as long as $\gamma \leq v$) but not the value itself. We investigate the query complexity of learning the optimal threshold in the latent space: the number of queries that are needed to learn a threshold with an expected reward at most $\eps$ smaller than the optimum.  

\paragraph{Our contributions}
The first contribution in this work is an impossibility result for this active learning problem.
We prove that the query complexity can be infinitely large \emph{even} when the reward function is monotone with respect to the threshold and the value. Our technique is built upon the idea of ``needle in the haystack". Intuitively, a higher threshold $\gamma$ gives a higher reward $g(\gamma, v)$ because of the monotonicity but decreases the probability of getting a reward which can only be achieved if $\gamma \le v$. This tradeoff allows us to construct an interval that has an equal expected utility. We find that if the reward function has a discontinuous bump at some point in the interval and the value distribution has a mass exactly at the same point, then the expected utility at this point will be constantly higher than the equal expected utility. Therefore we can hide the unique maximum (``needle'') in an equal utility interval (``haystack''), which makes the learner need infinite queries to learn the optimal threshold.

Our second contribution is a series of positive results with \emph{tight} query complexity up to logarithmic factors. We consider two special cases with common and minor assumptions: (1) the reward function is monotone and the CDF of value distribution is Lipschitz and (2) the reward function is right-Lipschitz. With each of the two assumptions,
we can apply a discretization technique to prove an $\tilde O(\frac{1}{\eps^3})$ upper bound on the query complexity of the threshold learning problem. 
We also give a matching lower bound, which is technically more challenging: at least $\Omega(\frac{1}{\eps^3})$ queries are needed to find an $\eps$-optimal threshold, even if the reward function is both monotone and Lipschitz and the value distribution CDF is Lipschitz.
%
Previous papers like \cite{kleinberg2003value,leme2023pricing} do not require the value distribution to be Lipschitz, which makes it much easier to construct a value distribution with point mass to prove lower bounds. To prove our lower bound with strong constraints on the reward function and distribution, we provide a novel construction of value distribution by careful perturbation of a smooth distribution. 
We summarize our main results in Table~\ref{tab:my-table}.

Finally, we extend the threshold learning problem to an online learning setting. Relating this problem to continuous-armed Lipschitz bandit, and using the aforementioned query complexity lower bound, we show a tight $\Theta(T^{2/3})$ regret bound for the online learning problem. 
\begin{table}[ht!]
\centering
\caption{Our results on the query complexity of learning optimal threshold. Rows correspond to reward functions and columns to value distribution. $\tilde{O}(\cdot)$ ignores poly-logarithmic factors.}
\label{tab:my-table}
\begin{center}
\begin{tabular}{|c|cc|cc|}
\hline
\multirow{2}{*}{Reward / Value} & \multicolumn{2}{c|}{Lipschitz}                                                                                                                                                                                                                              & \multicolumn{2}{c|}{General}                                                                                                                                                                           \\ \cline{2-5} 
& \multicolumn{1}{c|}{lower bound}                                                                                               & upper bound                                                                                                             & \multicolumn{1}{c|}{lower bound}                                                                              & upper bound                                                                            \\ \hline
Monotone                        & \multicolumn{1}{c|}{\multirow{2}{*}{\begin{tabular}[c]{@{}c@{}}$\Omega(\frac{1}{\varepsilon^3})$\\ \Cref{lower}\end{tabular}}} & \multirow{2}{*}{\begin{tabular}[c]{@{}c@{}}$\tilde O(\frac{1}{\varepsilon^3})$\\ \Cref{MONO+LIP}, \ref{Lipupper}\end{tabular}} & \multicolumn{2}{c|}{\begin{tabular}[c]{@{}c@{}}Infinite \\ \Cref{impossibility}\end{tabular}}                                                                                                          \\ \cline{1-1} \cline{4-5} 
Right-Lipschitz             & \multicolumn{1}{c|}{}                                                                                                          &                                                                                                                         & \multicolumn{1}{c|}{\begin{tabular}[c]{@{}c@{}}$\Omega(\frac{1}{\varepsilon^3})$\\ \Cref{lower}\end{tabular}} & \begin{tabular}[c]{@{}c@{}}$\tilde{O}(\frac{1}{\varepsilon^3})$\\ \Cref{Lipupper}\end{tabular} \\ \hline
\end{tabular}
\end{center}
\end{table}


\subsection{Related Works}
The most related work is the pricing query complexity of revenue maximization by \cite{leme2023pricing,leme2023description}. They consider the problem of how many queries are required to find an approximately optimal reserve price in the posted pricing mechanism, which is a strictly special case of our model by assuming $g(\gamma,v)=\gamma$. Our work also relates to previous work about reserve price optimization in online learning settings, e.g., \cite{cesa2014regret,feng2021reserve}, who consider revenue maximization in the online learning setting, where the learner can control the reserve price at each round. Our work is loosely related to the sample complexity of revenue maximization (e.g., \cite{cole2014sample, morgenstern2015pseudo, gonczarowski2017efficient, guo2019settling, brustle_multi-item_2020}). They focus on learning near-optimal mechanisms, which lies in the PAC learning style framework. Whereas, our work characterizes the query complexity in the active learning scenario. 



For the online setting, our work is most related to the Lipschitz bandit problem, which was first studied by \cite{agrawal1995continuum}. Once we have Lipschitzness, there is a standard discretization method to get the desired upper bound and it is widely used in online settings. See \cite{magureanu2014lipschitz,kleinberg2019bandits,haghtalab2022smoothed}. The upper bound of our online results follows this standard method, but the lower bound relies on our offline results and is different from previous continuous-armed Lipschitz bandit problems. Several recent works study multi-armed bandit problems with censored feedback. 
For example, \cite{abernethy_threshold_2016} study a bandit problem where the learner obtains a reward of 1 if the realized sample associated with the pulled arm is larger than an exogenously given threshold.
\cite{verma_censored_2019} study a resource allocation problem where the learner obtains reward only if the resources allocated to an arm exceed a certain threshold.
In \cite{guinet_effective_2022}, the reward of each arm is missing with some probability at every pull.
These models are significantly different from ours, hence the results are not comparable. 

Censored/Truncated data are also widely studied in statistical analysis. For example, a classical problem is truncated linear regression, which has remained a challenge since the early works of \cite{tobin1958estimation,amemiya1973regression,breen1996regression}. Recently, \cite{pmlr-v99-daskalakis19a} provided a computationally and statistically efficient solution to this problem. Statistics problems with known or unknown truncated sets have received much attention recently (e.g. \cite{daskalakis2018efficient,kontonis2019efficient}). For more knowledge of this field, the reader can turn to the textbook of \cite{little_statistical_2020}). While this line of work studies passive learning settings where the censored dataset is given to the data analyst exogenously, we consider an active learning setting where the learner can choose how to censor each data point, and how to do that optimally.  

Technically, our work is inspired by a high-level idea called ``the needle in the haystack'', which was first proposed by \cite{auer1995gambling} and also occurred in recent works such as online learning about bilateral trade \cite{cesa2021regret, cesa2023bilateral}, first-price auction \cite{cesa2023role}, and graph feedback \cite{eldowa2023minimax}.  
Nevertheless, this idea is only high-level.  As we will show in the proofs, adopting this idea to prove our impossibility results and lower bounds is not straightforward and requires careful constructions. 

\section{Model}


We first define some notations.
For an integer $n$, $[n]$ denotes the set $\{1,2,...,n\}$. $\1_{(\cdot)}$ is the indicator function.
We slightly abuse notation to use $F$ to denote both a distribution and its cumulative distribution function (CDF). 

We study the query complexity of learning thresholds with latent values between a \emph{learner} and an \emph{agent}. The \emph{latent value} $v$ represents the agent's private value and is drawn from an \emph{unknown} distribution $F$ supported on $[0,1]$. In each query, the learner is allowed to choose a \emph{threshold} $\gamma \in [0,1]$; then with a fresh sample $v\sim F$, the learner gets reward feedback $b$ determined by the threshold $\gamma$ and the value $v$: 
\begin{equation}\label{eq:b-g-definition}
    b(\gamma, v)=
    \begin{cases}
        g(\gamma, v) & \text{if } v \ge \gamma\\
        0 & \text{if } v < \gamma
    \end{cases}
\end{equation}
where $g: [0,1]^2 \rightarrow [0,1]$ is an \emph{unknown reward function}.
The notation $b(\gamma) = b(\gamma, v)$ denotes a random reward with randomness $v\sim F$. 
The learner's goal is to learn an optimal threshold $\gamma^* \in [0,1]$ that maximizes its expected reward/utility $U(\gamma)$ defined as
\begin{equation}
    U(\gamma) \triangleq \E_{v\sim F}\big[ b(\gamma,v) \big] = \E_{v\sim F}\big[ g(\gamma,v)\cdot\1_{v\ge\gamma} \big].
\end{equation}
Typically, a higher threshold decreases the probability of getting a reward but gives a higher reward if the value exceeds the threshold.

Our model of learning optimal thresholds with latent value captures many interesting questions, as illustrated by the following examples.
\begin{example}[reserve price optimization]
A seller (learner) repeatedly sells a single item to a set of $n$ bidders. The seller first sets a reserve price (threshold) $\gamma$. Each
bidder $i$ then submits a bid $b_i$. The bidder with the highest bid larger than $\gamma$ wins the item and
pays their bid; if no bidder bids above $\gamma$, the item goes unallocated. Each bidder i has a private valuation $v_i\in[0,1]$ for the item, where each value $v_i$ is drawn independently (but not necessarily identically) from some unknown distribution.
 
If the seller adopts the first-price auction, only the highest bid matters for both allocation and payment. We denote the maximum value (latent value) $v^{(1)}\triangleq\max v_i$ and $v^{(1)}$ is drawn from a unknown distribution. Then we only consider this representative highest bidder (agent).\footnote{This reduction is proved to be without loss of generality in~\cite{feng2021reserve}.} The unknown bidding function (reward function) $g$ and $g(\gamma,v^{(1)})$ is the maximum bid when the reserve price is $\gamma$ and the maximum value is $v^{(1)}$.

If the seller adopts the second-price auction, only the second-highest value matters for both allocation and payment. 
We denote the second highest value (latent value) $v^{(2)}$ and $v^{(2)}$ is drawn from an unknown distribution. Similarly, we only need to consider the second-highest bidder (agent). Because bidders in the second price auction bid truthfully, it has a bidding function (reward function) $g$ with $g(\gamma,v^{(2)})=v^{(2)}$ for all $\gamma,v^{(2)}\in[0,1]$ and $v^{(2)}\ge\gamma$.

If the seller faces a single bidder (agent) and adopts a posted price auction, we have $g(\gamma,v)=\gamma$ as the bid when the reserve price is $\gamma$ and the bidder's value is $v$.
\end{example}

\begin{example}[crowdsourced data collection]
Data crowdsourcing platforms typically allow users (agent) to sign up and complete tasks in exchange for compensation. These tasks might involve answering questions, providing feedback, or rating products. The taskmasters (learner) need to decide how many questions should be included in each task or how detailed the feedback should be. We denote such difficulties (threshold) of the task $\gamma$. The users have individual willingness (latent value) $v$ to complete the task. The willingness follows a specific probability distribution, which is unknown to the taskmasters or platform. Typically, a more difficult task decreases the probability of getting feedback from the users but gives a higher reward if the users are willing to complete the tasks because more detailed information can be included when using a more difficult task. We use the notation $g(\gamma, v)$ to represent the reward when the difficulty of the task is $\gamma$ and the user's willingness is $v$.
\end{example}

\begin{example}[hiring bar]
A company (learner) sets a predefined bar (threshold) $\gamma$ for candidate admission. These candidates (agents) have individual measurements (latent value) $v$, which reflects their inherent ability. They will apply to the company if and only if they think of themselves as qualified, namely, $v \ge \gamma$. The measurements follow a specific probability distribution, which is unknown to the company. A candidate with a measurement $v$ admitted with a hiring bar $\gamma$ produces an output (reward) $g(\gamma,v)$ to the company.
\end{example}


We assume that the value distribution $F$ belongs to some class $\mathcal{C}$, and the reward function $g$ belongs to some class $\mathcal{G}$.  The classes $\mathcal{C}$ and $\mathcal{G}$ are known to the learner.  The learner makes $m$ queries adaptively and then outputs a threshold $\hat \gamma \in [0, 1]$ according to some algorithm $\mathcal{A}$, namely, $\gamma_{t} = \mathcal{A}(\gamma_1, b_1, \ldots, \gamma_{t-1}, b_{t-1})$, $b_t = b(\gamma_t, v_t)$, $\forall t\in[m]$, and $\hat{\gamma} = \mathcal A(\gamma_1, b_1, \ldots, \gamma_m, b_m)$. 

\begin{definition}[$(\eps,\delta)$-estimator]
\emph{An $(\eps,\delta)$-estimator (for $\mathcal C$ and $\mathcal G$)} is an algorithm $\mathcal{A}$ that, for any $F\in\mathcal{C}, g\in\mathcal{G}$, can output a $\hat \gamma$ satisfying $U(\hat{\gamma}) \ge U(\gamma^*) - \eps$ with probability at least $1-\delta$ using $m$ queries (where the randomness is from $v_1, \ldots, v_m\sim F$ and the internal randomness of $\mathcal A$).
\end{definition}

\begin{definition}[query complexity]
Given $\mathcal{C}$, $\mathcal{G}$, for any $\eps>0$ and $\delta\in(0,1)$, the \emph{query complexity} $\QUERYCPLX_{\mathcal{C},\mathcal{G}}(\eps,\delta)$ is the minimum integer $m$ for which there exists an $(\eps,\delta)$-estimator. 
\end{definition}

The query complexity depends on both the value distribution class $\mathcal{C}$ and the reward function class $\mathcal{G}$.
In this work, we will consider two natural classes of value distributions: (1) $\mathcal{C}_{\textsc{ALL}}$, the set of all distributions supported on $[0, 1]$; 
(2) $\mathcal{C}_{\textsc{LIP}}$, the set of distributions on $[0, 1]$ whose CDF $F$ is Lipschitz continuous. 
And we consider two types of reward functions: \emph{monotone} and \emph{right-Lipschitz} (with respect to $\gamma$). 
Specifically, for any fixed $v\in[0, 1]$, define \emph{projection} $g_v\triangleq g(\cdot, v)$, which is a one dimensional function of $\gamma \in [0, v]$.
Let $\mathcal{G}_{\textsc{MONO}}$ be the set of reward functions $g$ whose projection $g_v$ is weekly increasing (w.r.t.$\gamma$) for all $v \in [0, 1]$.
For the Lipschitzness, we define:
\begin{definition}[Lipschitzness]
A one dimensional function $f$ is
\begin{compactitem}
    \item \emph{($L$-)left-Lipschitz}, if for all $x, y\in\mathrm{dom}(f)$ with $x \le y$, \, $f(y) - f(x) \ge -L(y-x)$.
    \item \emph{($L$-)right-Lipschitz}, if for all $x, y\in\mathrm{dom}(f)$ with $x \le y$, \, $f(y) - f(x) \le L(y-x)$.
    \item \emph{one-sided ($L$-)Lipschitz}, if it is either ($L$-)left-Lipschitz or ($L-$)righ-Lipshitz.
    \item \emph{($L$-)Lipschitz}, if it is both left- and righ-Lipshitz.
\end{compactitem} 
\end{definition}
Let $\mathcal{G}_{\textsc{right-LIP}}$ ($\mathcal{G}_{\textsc{LIP}}$) be the set of reward functions $g$ whose projection $g_v$ is right-Lipschitz (Lipschitz) for all $v\in[0, 1]$. Monotonicity and right Lipschitzness are natural assumptions of the reward functions. In the above examples, the rewards (quality of the hire, revenue, and the quality of collected data) are all \emph{weakly increasing} with respect to the thresholds (hiring bar, reserve price, and the difficulty of the requirement). For right-Lipschitz functions, one can see \cite{duetting2023optimal} for some practical examples.

\section{Impossibility result: monotone reward function and general value distribution}\label{Neg}

In this section, we investigate the query complexity of learning the optimal threshold for general value distributions.
We show that even when the reward function $g$ is monotone with respect to both the threshold $\gamma$ and the latent value $v$, an $(\eps, \delta)$-estimator cannot be learned with finitely many queries, even for a constant $\eps = \frac{1}{8}$. 

\begin{theorem}\label{impossibility}
   For any $\delta \in (0,1)$ and $\eps \le \frac{1}{8}$, the query complexity $\QC_{\mathcal{G}_{\textsc{MONO}}, \mathcal{C}_{\textsc{ALL}}}(\eps, \delta)$ is infinite.
\end{theorem}

\paragraph{High-Level Ideas}
To prove the theorem, we carefully construct a set of pairs of monotone reward function and value distribution $(g_{\alpha}, F_{\alpha}) \in \+G_{\textsc{MONO}} \times  \+C_{\textsc{ALL}}$, parameterized by $\alpha \in (\frac{1}{2}, \frac{9}{16})$, such that that the utility function $U(\gamma)$ has a unique maximum point at $\gamma^* = \alpha$ and $U(\gamma) < U(\alpha) - \varepsilon$ for any $\gamma \ne \alpha$. To learn an $\varepsilon$-approximately optimal threshold, the learner must find the point $\gamma^* = \alpha$.
But there are infinitely many possible values for $\alpha \in (\frac{1}{2}, \frac{9}{16})$, and our construction ensures that the learner cannot determine the exact value of $\alpha$ from the feedback of finitely many queries. 


\begin{proof}
Fix any $\alpha\in(\frac{1}{2},\frac{9}{16})$. We define value distribution $F_{\alpha}$ with the following CDF: 
\begin{equation*}
F_{\alpha}(v)=
    \begin{cases}
        0 & v \in [0, \frac{1}{2})\\
        \frac{1}{2}-\frac{3}{16v} & v \in [\frac{1}{2},\alpha)\\
        \frac{1}{2} & v \in [\alpha, 1)\\
        1  & v=1. \\
    \end{cases}
\end{equation*}   
More specifically, $F_\alpha$ consists of the following four parts: (1) A point mass at $\frac{1}{2}$: $\sP(v = \frac{1}{2}) = \frac{1}{8}$; (2) Continuous CDF over the interval $(\frac{1}{2}, \alpha)$: $F_\alpha(x) = \frac{1}{2} - \frac{3}{16 x}$; (3) A point mass at $\alpha$: $\sP(v = \alpha) = \frac{3}{16\alpha}$; (4) A point mass at $1$: $\sP(v = 1) = \frac{1}{2}$.
Then we define a reward function $g_\alpha$ that is monotone with respect to both the threshold $\gamma$ and the latent value $v$:
\begin{equation*}
    g_{\alpha}(\gamma,v)=
    \begin{cases}
        \gamma     & \gamma \in [0, \alpha), v \in [0, \frac{15}{16}) \\
        v-\frac{3}{8}  & \gamma \in [0, \alpha), v \in [\frac{15}{16}, 1] \\
        \gamma & \gamma \in [\alpha, 1], v \in [0, \frac{15}{16}) \\
        v & \gamma \in [\alpha, 1], v \in [\frac{15}{16}, 1]. 
    \end{cases}
\end{equation*}
Now we can compute the expected utility  $U_{\alpha}(\gamma)$  when the learner chooses the threshold $\gamma \in [0,1]$ for the reward function $g_{\alpha}$ and value distribution $F_{\alpha}$. We have
\begin{compactitem}
    \item when $\gamma \in [0, \frac{1}{2}]$, the utility is $U_{\alpha}(\gamma) = \frac{1}{2}\gamma+\frac{5}{16}<\frac{9}{16}$; 
    \item when $\gamma \in (\frac{1}{2}, \alpha)$, the utility is $U_{\alpha}(\gamma) = (F(\alpha)-F(\gamma))\gamma+\frac{5}{16}=\frac{1}{2}$;
    \item when $\gamma = \alpha$, the utility is $U_{\alpha}(\gamma)=\alpha \cdot \sP[v=\alpha] + \frac{1}{2} = \frac{11}{16}$; 
    \item when $\gamma \in (\alpha, 1]$, the utility is $U_{\alpha}(\gamma)=\frac{1}{2}$. 
\end{compactitem}

Therefore, $U_{\alpha}(\gamma)$ is maximized at $\gamma^* = \alpha$, and for any $\gamma\neq\alpha$, it holds that $U_{\alpha}(\alpha) - U_{\alpha}(\gamma)> \frac{1}{8}$. To approximate the optimal utility within $\eps = \frac{1}{8}$ error, the learner must learn the exact value of $\alpha$. However, the following claim shows that when a learner chooses any threshold $\gamma \in [0,1]$, it only observes censored feedback in  $\{\gamma, \frac{5}{8}, 1, 0\}$.
\begin{claim}\label{claim:infinite lower bound}
    When the learner chooses $\gamma \in [0,1]$, it observes feedback in $\{ 0, \gamma, \frac{5}{8}, 1\}$.
\end{claim}
\begin{proof}
    We consider two cases. If the learner chooses a threshold $\gamma < \alpha$, the learner only receives feedback in $\{ \frac{5}{8}, \gamma, 0\}$ depending on the latent value $v$: (1) if $v \ge \frac{15}{16}$, the learner receives $g_\alpha(\gamma, v) = v - \frac{3}{8} = 1 - \frac{3}{8} = \frac{5}{8}$; (2) if $\gamma \le v_t < \frac{15}{16}$, the learner receives $g_\alpha(\gamma, v) = \gamma$; (3) if $v < \gamma$, the learner receives $0$. Similarly, if the learner chooses a threshold $\gamma \ge \alpha$, the learner only receives feedback in $\{0,1, \gamma\}$: (1) if $v \ge \frac{15}{16}$, the learner receives $g_\alpha(\gamma, v) = v = 1$; (2) if $\gamma \le v < \frac{15}{16}$, the learner receives $g_\alpha(\gamma, v) = \gamma$; (3) if $v < \gamma$, the learner receives $0$.  This proves the claim.
\end{proof}
Note that the above results holds for all $\alpha \in (\frac{1}{2}, \frac{9}{16})$. By properties of $U_\alpha(\gamma)$ and \Cref{claim:infinite lower bound}, a learner is not able to output the exact value of $\alpha$ using finite queries with feedback only in $\{0, \gamma, \frac{5}{8}, 1\}$ against infinitely many pairs of $(F_\alpha, g_\alpha)$ for $\alpha \in (\frac{1}{2}, \frac{9}{16})$.  
\end{proof}

\section{Tight Query Complexity for the Lipschitz Cases}\label{Pos}
\subsection{Monotone reward function and Lipschitz value distribution}
The negative result in \Cref{Neg} implies that we need more assumptions on the reward function $g$ or the value distribution $F$ for the learner to learn the optimal threshold in finite queries. In this subsection, we keep assuming that $\mathcal{G}$ is the class of monotone functions w.r.t.~$\gamma$ and further assume that $\mathcal{C}$ is the class of Lipschitz distributions. 
We argue that the monotonicity of $g$ and the Lipschitzness of $F$  guarantee that the expected utility function $U$ is left-Lipschitz.

\begin{lemma}\label{mono}
    With $g\in \mathcal{G}_{\textsc{MONO}}$ and $F
    \in \mathcal{C}_{\textsc{LIP}}$, the expected utility function $U$ is left-Lipschitz.
\end{lemma}
\begin{proof}
The distribution $F$ is weakly differentiable because it is Lipschitz. Let $f$ be the weak derivative of $F$.  We have $f(v)\le L$ for all $v\in[0,1]$ because of Lipschitzness.
We rewrite the expected utility $U(\gamma)=\E_{v\sim F}\big[ g(\gamma,v)\cdot\1_{v\ge\gamma} \big]$ as $\int_{v=\gamma}^1 g(\gamma,v) f(v) dv$.
Then for any $0\le \gamma_1 \le \gamma_2 \le 1$,
 \begin{equation*}
 \begin{aligned}
         U(\gamma_2)-U(\gamma_1) &= \int_{v=\gamma_2}^1g(\gamma_2, v)f(v)dv \,- \int_{v=\gamma_1}^1 g(\gamma_1,v)f(v)dv\\
         &=\int_{v=\gamma_2}^1 \big( g(\gamma_2,v) - g(\gamma_1,v)\big) f(v)dv \,-\int_{v=\gamma_1}^{\gamma_2}g(\gamma_1,v)f(v)dv\\
         &\ge 0 \,- \int_{v=\gamma_1}^{\gamma_2}g(\gamma_1,v)f(v)dv ~~~~~ \ge ~~~ -L(\gamma_2 - \gamma_1)\\
 \end{aligned}
 \end{equation*}
 where the first inequality holds because $g$ is monotone in $\gamma$, the second inequality holds because $g(\gamma,v)\le1$ and $f(v)\le L$ for all $\gamma,v\in[0,1]$.
\end{proof}

Then we show that $\Tilde{O}\InParentheses{\frac{L}{\eps^3}}$ queries are enough to learn the optimal threshold.
\begin{theorem}\label{MONO+LIP}
For $\mathcal{G}_{\textsc{MONO}}$ and $\mathcal{C}_{\textsc{LIP}}$, we have
\begin{equation*}
\QUERYCPLX_{\mathcal{G}_{\textsc{MONO}},\mathcal{C}_{\textsc{LIP}}}(\eps,\delta)\le O\InParentheses{\frac{L}{\eps^3}\log\frac{L}{\eps\delta}}. 
\end{equation*}
\end{theorem}
\begin{proof}
If the learner chooses a threshold $\gamma$ and makes $m$ queries with the same threshold $\gamma$, it will observe $m$ i.i.d.~samples $b_1(\gamma), b_2(\gamma), \ldots, b_m(\gamma)$. Here we use $b_i(\gamma)$ to denote the random variable $b(\gamma, v_i)$ where $v_i\sim F$. Let $G_{\gamma}$ be the CDF of the distribution of $b_i(\gamma)$. 
Let $\hat{G}_{\gamma}$ be the CDF of the empirical distribution: $\hat{G}_{\gamma}(x)=\sum_{i=1}^m\1_{b_i(\gamma)\le x}$. 
By the DKW inequality (Lemma~\ref{lem:DKW}), if the number of queries $m$ reaches $O(\frac{1}{\eps^2}\log\frac{1}{\delta'})$, then with probability at least $1-\delta'$ we have 
\begin{equation*}
    \InAbs{G_{\gamma}(x)-\hat{G}_{\gamma}(x)}\le\eps, ~~~~ \forall x\in\mathbb{R}.  
\end{equation*}
By Tonelli's theorem, $
    U(\gamma)=\E_{b(\gamma)\sim G_{\gamma}}[b(\gamma)]=\int_0^1\sP[b(\gamma)>t]dt=\int_0^1\big(1-G_{\gamma}(t)\big)dt$. 

Define $\hat{U}(\gamma)=\int_0^1\big(1-\hat{G}_{\gamma}(t)\big)dt$. 
Then we have 
\begin{equation}\label{single-est}
    \InAbs{U(\gamma)-\hat{U}(\gamma)}=\InAbs{\int_0^1\big(G_{\gamma}(t)-\hat{G}_{\gamma}(t)\big)dt} \le\eps. 
\end{equation}
This means that, after $O\InParentheses{\frac{1}{\eps^2}\log\frac{1}{\delta'}}$ queries at the same point $\gamma$, we can learn the expected utility of threshold $\gamma$ in $\eps$ additive error with probability at least $1-\delta'$.

For any $\eps>0$, let's consider all the multiples of $\frac{\eps}{L}$ in $[0,1]$, $\Gamma\triangleq\{\frac{k\eps}{L}:k\in\sN, k\le\frac{L}{\eps}\}$.
We make $O\InParentheses{\frac{1}{\eps^2}\log\frac{|\Gamma|}{\delta}}$ queries for each threshold $\gamma = \frac{k\eps}{L}$ in $\Gamma$.  By a union bound, we have with probability at least $1 - |\Gamma|\cdot\frac{\delta}{|\Gamma|} = 1-\delta$, the estimate $\hat U(\gamma)$ satisfies $|\hat U(\gamma) - U(\gamma)| \le \eps$ for every threshold in $\Gamma$.  Since $|\Gamma| = O(\frac{L}{\eps})$, this uses $|\Gamma| \cdot O\InParentheses{\frac{1}{\eps^2}\log\frac{|\Gamma|}{\delta}} = O\InParentheses{\frac{L}{\eps^3}\log\frac{L}{\eps\delta}}$ queries in total.  

Then, let $\gamma^* = \argmax_{\gamma\in[0, 1]} U(\gamma)$ be the optimal threshold in $[0, 1]$ for $U(\gamma)$ and $\hat{\gamma}^* = \argmax_{\gamma\in\Gamma}\hat{U}(\gamma)$ be the optimal threshold in $\Gamma$ for $\hat U(\gamma)$.
And let $\gamma_r=\frac{\eps}{L}\lceil\frac{L\gamma^*}{\eps}\rceil$, so $\gamma_r\in\Gamma$ and $0<\gamma_r-\gamma^*<\frac{\eps}{L}$.
Then, we have the following chain of inequalities: 
\begin{equation*}
\begin{aligned}
        U(\hat{\gamma}^*) & ~\stackrel{\text{\cref{single-est}}}{\ge}~ \hat{U}(\hat{\gamma}^*)-\eps ~\stackrel{\text{Definition of $\hat{\gamma}^*$}}{\ge}~ \hat{U}(\gamma_r)-\eps 
        ~\stackrel{\text{\cref{single-est}}}{\ge}~ U(\gamma_r)-2\eps\\
        & \stackrel{\text{\cref{mono}}}{\ge} ~  U(\gamma^*)-2\eps-L(\gamma_r-\gamma^*) ~~~ \ge ~~~ U(\gamma^*)-3\eps. 
\end{aligned}
\end{equation*}
This means that we have obtained a $3\eps$-optimal threshold $\hat{\gamma}^*$. 
%
\end{proof}


\begin{remark}
The above algorithm needs to know the Lipschitz constant $L$.  In Appendix \ref{app:L} we provide another algorithm that does not need to know $L$ and has a better query complexity of $\tilde O\InParentheses{\frac{1}{\eps^3}}$.    
\end{remark}

\subsection{Right-Lipschitz reward function and general value distribution}
The argument in the previous section shows that the one-sided Lipschitzness of the expected utility function is crucial to learning an optimal threshold with $\tilde{O}\InParentheses{\frac{L}{\eps^3}}$ query complexity. In this section, we further show that the expected utility function is right-Lipschitz when $\mathcal{G}$ is the class of right-Lipschitz continuous reward function and $\mathcal{C}$ is the class of general distribution. 

\begin{lemma}\label{Liplemma}
    Suppose reward function $g\in\mathcal{G}_{\textsc{right-LIP}}$, then the expected reward function $U$ is right-Lipschitz continuous.
\end{lemma}
The proof of \cref{Liplemma} is similar to \cref{mono} and we leave it to the appendix. Then we can use the same method as \cref{MONO+LIP} to provide an upper bound.
\begin{theorem}\label{Lipupper}
    For the class of right-Lipschitz reward functions $\mathcal{G}_{\textsc{right-LIP}}$ and general distributions $\mathcal{C}_{\textsc{ALL}}$, $\varepsilon > 0$, $\delta \in [0,1]$, we have
    \begin{equation*}
    \QUERYCPLX_{\mathcal{G}_{\textsc{LIP}},\mathcal{C}_{\textsc{LIP}}}(\eps, \delta) \le \QUERYCPLX_{\mathcal{G}_{\textsc{right-LIP}},\mathcal{C}_{\textsc{ALL}}}(\eps, \delta)\le O\InParentheses{\frac{L}{\eps^3}\log\frac{L}{\eps\delta}}.
    \end{equation*}
\end{theorem}

\subsection{Lower Bound}
In this section, we prove that even if $\mathcal{G}$ is the class of reward functions that are both Lipschitz and monotone w.r.t.~$\gamma$ and $\mathcal{C}$ is the class of Lipschitz distributions, $\Omega(\frac{1}{\eps^3})$ queries are needed to learn the optimal utility within $\eps$ error. This is a uniformly matching lower bound for all the upper bounds in this paper.

\begin{theorem}\label{lower}
    For $\mathcal{G}=\mathcal{G}_{\textsc{LIP}}\cap\mathcal{G}_{\textsc{MONO}}$, $\mathcal{C}_{\textsc{LIP}}$, we have
    \begin{equation*}
        \QUERYCPLX_{\mathcal{G},\mathcal{C}_{\textsc{LIP}}}(\eps, \delta)\ge\Omega\InParentheses{\frac{1}{\eps^3}+\frac{1}{\eps^2}\log\frac{1}{\delta}}. 
    \end{equation*}
\end{theorem}

\paragraph{High-level ideas}
To prove the theorem, we construct a Lipschitz value distribution and carefully perturb it on a $O(\eps)$ interval. The base value distribution leads to any $\gamma\in [\frac{1}{3},\frac{1}{2}]$ being the optimal threshold. However, the perturbed distribution leads to a unique optimal threshold $\gamma^*\in[\frac{1}{3},\frac{1}{2}]$. To learn an $\eps$-approximately optimal threshold, the learner needs to distinguish the base distribution and $O(\frac{1}{\eps})$ perturbed distributions. It can be proved that $\Omega\InParentheses{\frac{1}{\eps^2}\log\frac{1}{\delta}}$ queries are needed to distinguish the base distribution and one perturbed distribution, which leads to our desired lower bound.  This idea is significantly different from previous works that constructed discrete distributions (e.g., \cite{kleinberg2003value, leme2023pricing}).

\begin{proof}
    Let $g(\gamma,v)=\gamma$ for any $\gamma,v\in[0,1]$ such that $v\ge\gamma$. It is not difficult to verify that $g$ is Lipschitz continuous and monotone w.r.t.~$\gamma$. In this case, the expected utility can be written in a simple form $U(\gamma)=\gamma(1-F(\gamma))$.

    Consider the following Lipschitz continuous distribution $F_0$: 
    \begin{equation*}
    F_0(v)=
        \begin{cases}
            \frac{3}{4}v & v\in[0, \frac{1}{3})\\
            1-\frac{1}{4v} & v\in[\frac{1}{3}, \frac{1}{2}]\\
            v & v\in(\frac{1}{2}, 1]. 
        \end{cases}
    \end{equation*}

    This distribution leads to the following expected utility function:
    \begin{equation*}
    U_0(\gamma)=
        \begin{cases}
            \gamma(1-\frac{3}{4}\gamma) & \gamma\in[0,\frac{1}{3})\\
            \gamma(1-(1-\frac{1}{4\gamma}))=\frac{1}{4} & \gamma\in[\frac{1}{3},\frac{1}{2}]\\
            \gamma(1-\gamma) &\gamma\in(\frac{1}{2},1]. 
        \end{cases}
    \end{equation*}
    Because $\gamma(1-\frac{3}{4}\gamma)$ is increasing on interval $[0,\frac{1}{3})$ and $\gamma(1-\gamma)$ is decreasing on interval $(\frac{1}{2},1]$, $U_0(\gamma)<U_0(\frac{1}{3})=\frac{1}{4}$ when $\gamma\in[0,\frac{1}{3})$ and $U_0(\gamma)<U_0(\frac{1}{2})=\frac{1}{4}$ when $\gamma\in(\frac{1}{2},1]$.
    In other words, the expected utility function reaches maximum value if and only if $\gamma\in[\frac{1}{3},\frac{1}{2}]$. There is a "plateau" on the utility curve. (See Figure~\ref{fig:enter-label}.)

    Next, we perturb the distribution $F_0$ to obtain another distribution $F_{\omega, \eps}$, whose expected utility function will only have one maximum point rather than the interval $[\frac{1}{3},\frac{1}{2}]$.
    So to estimate the optimal utility, the learner must distinguish $F_0$ and a class of perturbed distributions. However, by carefully designing the perturbation, the difference between $F_0$ and perturbed distributions is small enough to lead to the desired lower bound.

    Let $\Xi=\{(w,\eps)\in[0,1]^2,w-3\eps \ge \frac{1}{3},w+3\eps\le\frac{1}{2}\}$.  For any $(w,\eps)\in\Xi$, let 
        $h_{w,\eps}(v)=\1_{v\in(w,w+3\eps]}-\1_{v\in[w-3\eps,w)}.$
    Notice that $F_0$ has the following probability distribution function $f_0$:
    \begin{equation*}
        f_0(v)=
        \begin{cases}
            \frac{3}{4} & v\in[0,\frac{1}{3})\\
            \frac{1}{4v^2} & v\in[\frac{1}{3},\frac{1}{2}]\\
            1 & v\in(\frac{1}{2},1].
        \end{cases}
    \end{equation*}
    So $f_0(v)\ge 1$ when $v\in[\frac{1}{3},\frac{1}{2}]$.  Define $f_{w,\eps}(v)=f_0(v)+h_{w,\eps}(v)$, then we have $f_{w,\eps}(v)\ge 0$ for any $v\in[0,1]$. And $\int_{v=0}^1f_{w,\eps}(v)=1$, so $f_{w,\eps}$ is a valid probability density function. Let $F_{w,\eps}$ be the corresponding cumulative distribution function. Note that $F_{w,\eps}$ is $L$-Lipschitz for any constant $L\ge\frac{13}{4}$. By definition, $F_{w,\eps}(v)=\int_0^v f_{\omega, \eps}(t) dt = \int_{0}^v\big(f_0(t)+h_{w,\eps}(t)\big)dt = F_0(v) + \int_{0}^vh_{w,\eps}(t)dt$. 
    \begin{claim}\label{claim:F_w_eps}(See proof in \Cref{app:proof:claim:F_w_eps}) The CDF function $F_{w,\varepsilon}$ is 
    \begin{equation*}
        F_{w,\eps}(v)=
        \begin{cases}
            F_0(v)-(v-w+3\eps) & v\in[w-3\eps,w)\\
            F_0(v)-(w+3\eps-v) & v\in[w,w+3\eps]\\
            F_0(v) & \mathrm{otherwise}. 
        \end{cases}
    \end{equation*}
    \end{claim}

\cref{claim:F_w_eps} means that $F_{w,\eps}$ and $F_0$ are nearly the same except on the $6\eps$-long interval $[w-3\eps,w+3\eps]$.

Let $U_{w,\eps}(\gamma)$ be the expected utility function when the agent's value distribution is $F_{w,\eps}$.
\begin{equation*}
    U_{w,\eps}(\gamma)-U_0(\gamma)=\gamma\big(F_0(\gamma)-F_{w,\eps}(\gamma)\big)=
    \begin{cases}
        \gamma(\gamma-(w-3\eps)) & \gamma\in[w-3\eps,w)\\
        \gamma(w+3\eps-\gamma) & \gamma\in[w,w+3\eps]\\
        0 & \mathrm{otherwise}. 
    \end{cases}
\end{equation*}
\begin{figure}
    \centering
    \includegraphics[width=60mm]{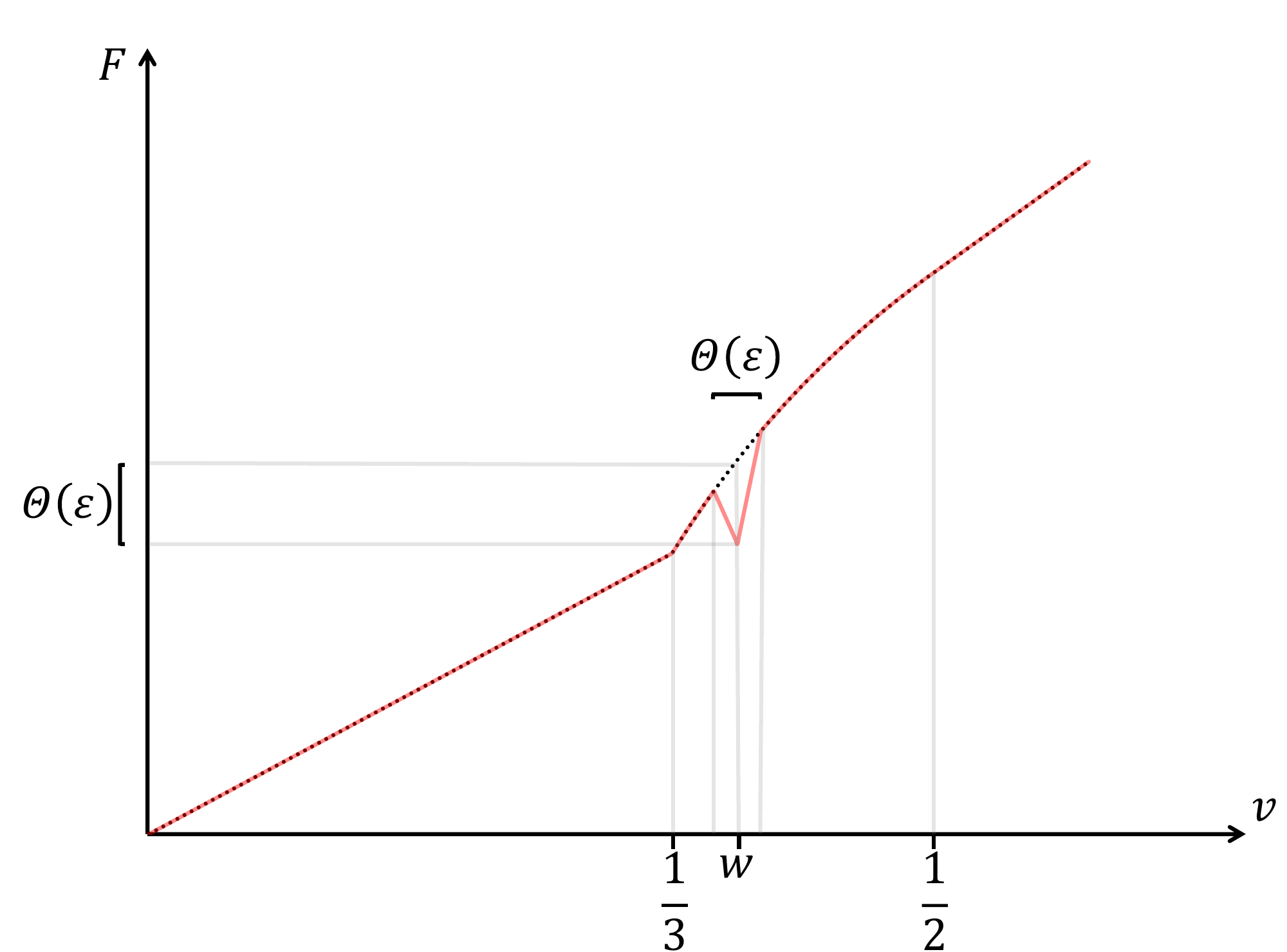}
    \includegraphics[width=60mm]{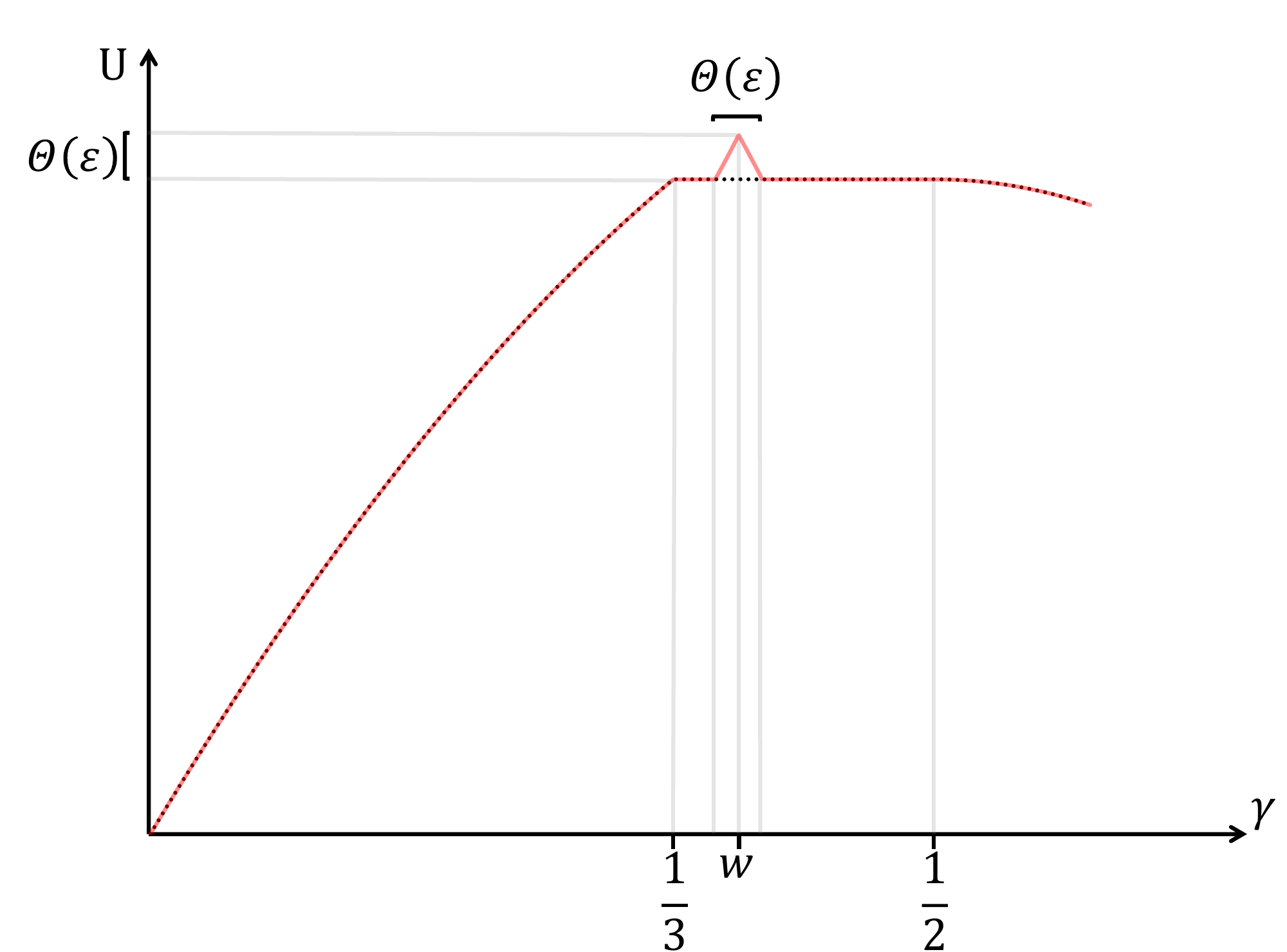}
    \caption{Left: The base distribution $F$ (black, dotted) and the perturbed distribution $F_{w,\gamma}$ (red solid), which moves one unit of mass from interval $[w-3\eps,w]$ to interval $[w,w+3\eps]$. Right: The corresponding qualitative plots
    of $\gamma\mapsto U(\gamma)$ (black, dotted) and $\gamma\mapsto U_{w,\gamma}(\gamma)$ (red, solid).}
    \label{fig:enter-label}
\end{figure}
$U_{w,\eps}$ has a unique maximum point at $\gamma = w$, with $U_{w,\eps}(w)=U_0(w)+w\eps=\frac{1}{4}+3w\eps>\frac{1}{4}+\eps$, because $\gamma(\gamma-(w-3\eps))$ is increasing in $[w-3\eps,w)$ and $\gamma(w+3\eps-\gamma)$ is decreasing in $[w,w+3\eps]$.

Let $w_i=\frac{1}{3}+3(2i-1)\eps, i\in\{1,2,...,n=\lfloor\frac{1}{36\eps}\rfloor\}$. Note that $(w_i, \eps)\in\Xi$ for all $i\in[n]$.
It can be proved that distinguishing distributions $F_0$ and $F_{w_i,\eps}$ requires $\Omega\InParentheses{\frac{1}{\eps^2}\log\frac{1}{\delta}}$ queries in the interval $[w_i-3\eps, w_i+3\eps]$ (see \Cref{app:dis}).  Queries not in $[w_i-3\eps, w_i+3\eps]$ do not help to distinguish $F_0$ and $F_{w_i,\eps}$ because their CDFs are the same outside of the range $[w_i-3\eps, w_i+3\eps]$. 
\begin{lemma}\label{dis}
For any $i\in[n]$, $\Omega\InParentheses{\frac{1}{\eps^2}\log\frac{1}{\delta}}$ queries in interval $[w_i-3\eps, w_i+3\eps]$ are needed to distinguish $F_0$ and $F_{w_i,\eps}$.
\end{lemma}
Now consider a setting where the underlying value distribution is $F_{w_i,\eps}$ for $i$ uniformly drawn from $\{1, 2, \ldots, n\}$. Finding out the optimal utility and the corresponding threshold is equivalent to finding out the underlying value distribution. On one side, fixing $\delta$, for each $i\in[n]$ we need $\Omega\InParentheses{\frac{1}{\eps^2}}$ queries to distinguish $F_{w_i,\eps}$ and $F_0$ and all these queries must lie in the interval $I_i=[w_i-3\eps,w_i+3\eps]$. Since these intervals $\{I_i\}_{i \in [n]}$ are disjoint, and the learner must make  $\Omega\InParentheses{\frac{1}{\eps^2}}$ queries in each interval $I_i$, this leads to $n\cdot \Omega\InParentheses{\frac{1}{\eps^2}} = \Omega\InParentheses{\frac{1}{\eps^3}}$ queries in total. On the other side, for any $\delta\in(0,1)$, $\Omega\InParentheses{\frac{1}{\eps^2}\log\frac{1}{\delta}}$ are needed to distinguish the underlying value distribution and other distributions with probability at least $1-\delta$. Combining these two lower bounds, $\Omega\InParentheses{\frac{1}{\eps^3}+\frac{1}{\eps^2}\log\frac{1}{\delta}}$ queries are needed to learn the optimal threshold in $\eps$ additive error with probability at least $1-\delta$.
\end{proof}


\begin{corollary}\label{cor}
$\QUERYCPLX_{\mathcal{G}_{\textsc{LIP}},\mathcal{C}_{\textsc{LIP}}}(\eps,\delta)\ge\Omega\InParentheses{\frac{1}{\eps^3}+\frac{1}{\eps^2}\log\frac{1}{\delta}}$, $\QUERYCPLX_{\mathcal{G}_{\textsc{MONO}}, \mathcal{C}_{\textsc{LIP}}}(\eps, \delta)\ge\Omega\InParentheses{\frac{1}{\eps^3}+\frac{1}{\eps^2}\log\frac{1}{\delta}}$. 
\end{corollary}

\section{Online learning}\label{online}
Previous sections studied the threshold learning problem in an offline query complexity model.  In this section, we consider an (adversarial) online learning model and show that the threshold learning problem in this setting can be solved with a tight $\Theta\InParentheses{T^{2/3}}$ regret in the Lipschitz case.  

\paragraph{Adversarial online threshold learning}
The learner and the agent interact for $T$ rounds.
At each round $t \in [T]$, the learner selects a threshold $\gamma_t\in[0, 1]$ and the agent realizes a value $v_t\sim F_t \in \+C$.
The learner then observes reward $b_t = b_t(\gamma_t, v_t) = g_t(\gamma_t, v_t)\1_{v_t\ge\gamma_t}$ as in \cref{eq:b-g-definition}, which depends on the unknown reward function $g_t \in \+G$ and the value $v_t$.
Unlike the query complexity model where the learner only cares about the quality of the final output threshold $\hat \gamma$, here, the learner cares about the total reward gained during the $T$ rounds: $\sum_{t=1}^T b_t(\gamma_t, v_t)$.  We compare this total reward against the best total reward the learner could have obtained using some fixed threshold in hindsight.  In other words, we measure the performance of the learner's algorithm $\+A$ by its \emph{reget}: 
\begin{equation}
    \mathrm{Reg}_T^{\+C, \+G}(\+A) = \sup_{\gamma\in[0,1]}\-E\InBrackets{\sum_{t=1}^Tb_t(\gamma,v_t)-\sum_{t=1}^Tb_t(\gamma_t,v_t)}.
\end{equation}
Unlike the query complexity model where the value distribution $F$ and reward function $g$ are fixed, in this online learning model we allow them to change over time. Moreover, they can be controlled by an \emph{adaptive adversary} who, given classes $\+C$ and $\+G$, at each round $t$ can arbitrarily choose an $F_t \in \+C$ and a $g_t \in \+G$ based on the history $(\gamma_1, v_1, \ldots, \gamma_{t-1}, v_{t-1})$. 

We have shown in \Cref{Neg} that learning an $\eps$-optimal threshold is impossible for monotone reward functions $\+G_{\textsc{MONO}}$ and general distributions $\+C_{\textsc{ALL}}$.  Similarly, it is impossible to obtain $o(1)$ regret in this case.  So, we focus on the two cases with finite query complexity: $\+S_1 = (\+G_{\textsc{right-LIP}}, \+C_{\textsc{ALL}})$, $\+S_2 = (\+G_{\textsc{MONO}}, \+C_{\textsc{LIP}})$. 
The following theorem shows that, for all these three cases, there exists a learning algorithm that achieves $\Theta(T^{2/3})$ regret, and this bound is tight. 

\begin{theorem}\label{noregret}
For the adversarial online threshold learning problem, there exists an algorithm $\+A$ that, for all the three environments $\+S_1 = (\+G_{\textsc{right-LIP}}, \+C_{\textsc{ALL}})$, $\+S_2 = (\+G_{\textsc{MONO}}, \+C_{\textsc{LIP}})$, achieves regret 
\[
\mathrm{Reg}_T^{\+S_i}(\+A) \leq O(T^{2/3}). 
\]
And for any algorithm $\+A$, there exists a fixed reward function $g \in \mathcal{G}_{\textsc{LIP}}\cap\mathcal{G}_{\textsc{MONO}}$ and a fixed distribution $F \in \mathcal{C}_{\textsc{LIP}}$ for which the regret of $\+A$ is at least $\mathrm{Reg}_T^{g, F}(\+A) \ge \Omega\InParentheses{T^{2/3}}$.
\end{theorem}

The proof idea is as follows. 
The lower bound $\Omega\InParentheses{T^{2/3}}$ follows from the $\Omega(\frac{1}{\eps^3})$ query complexity lower bound in \Cref{lower} by a standard online-to-batch conversion.
To prove the upper bound $O(T^{2/3})$, we note that under the three environments $\+S_1, \+S_2$, the expected reward function $U_t(\gamma) = \E_{v_t\sim F_t}[b(\gamma, v_t)]$ is one-sided Lipschitz in $\gamma$.
Therefore, we can treat the problem as a continuous-arm one-sided Lipschitz bandit problem, where each threshold $\gamma \in [0, 1]$ is an arm. This problem can be solved by discretizing the arm set and running a no-regret bandit algorithm for a finite arm set, e.g., Poly INF \cite{audibert_regret_2010}.  See details in \cref{app:proof:noregret}.

\bibliography{reference}
\bibliographystyle{plain}

\appendix
\section{Basic math}
\subsection{DKW Inequality}
The following well-known concentration inequality will be used in our proofs.
\begin{lemma}[Dvoretzky–Kiefer–Wolfowitz (DKW) inequality \cite{dvoretzky_asymptotic_1956, massart_tight_1990}]
\label{lem:DKW}
Let $X_1, X_2, \ldots, X_n$ be $n$ real-valued i.i.d random variables with cumulative distribution function $F$. Let $F_n = \sum_{i=1}^n\1_{X_i\le x}$ be the associated empirical distribution function. 
Then for all $\eps>0$,
\[
 \sP\InBrackets{\sup_{x\in\sR}\big(F_n(x)-F(x)\big)>\eps}<2e^{-2n\eps^2}. 
\]
\end{lemma}

\subsection{Properties of Hellinger Distance}
Then we review some useful properties of the Hellinger distance and total variation distance. First, the Hellinger distance gives upper bounds on the total variation distance:
\begin{fact}\label{HFact}
    Let $D_1$, $D_2$ be two distribution on $\+X$. Their total variance distance and Hellinger distance are $d_{\textsc{TV}}(D_1,D_2)$ and $d_{\textsc{H}}(D_1,D_2)$ respectively. We have
    \[
    1-d^2_{\textsc{TV}}(D_1,D_2)\ge(1-d^2_{\textsc{H}}(D_1,D_2))^2.
    \]
\end{fact}

The total variation distance has the following well-known property that upper bounds the
difference between the expected values of a function on two distributions:
\begin{fact}\label{Tfact}
    For any function $h:\+X\rightarrow[0,1],\InAbs{\E_{x\sim D_1}[h(x)]-\E_{x\sim D_2}[h(x)]}\le d_{\textsc{TV}}(D_1,D_2)$.
\end{fact}

Second, we use the following lemma to upper bound the squared Hellinger distance between two distributions that are close to each other. We use $D$ to specifically denote discrete distributions. We slightly abuse the notation by using $D$ to denote of PDF of distribution $D$.

\begin{lemma}\label{Hupperbound}
    Let $D_1$, $D_2$ be two distribution on $\+X$ satisfying $1-\eps\le\frac{D_2(x)}{D_1(x)}\le1+\eps$ for all $x\in\+X$. Then $d^2_{\textsc{H}}(D_1,D_2)\le\frac{1}{2}\eps^2$.
\end{lemma}
\begin{proof}
    By definition, 
    \[
    d^2_{\textsc{H}}(D_1,D_2)=\frac{1}{2}\sum_{x\in\+X}(\sqrt{D_1(x)}-\sqrt{D_2(x)})^2=\frac{1}{2}\sum_{x\in\+X}D_1(x)\InParentheses{1-\sqrt{\frac{D_2(x)}{D_1(x)}}}.
    \]
    If $D_2(x)<D_1(x)$, then we have $1-\sqrt{\frac{D_2(x)}{D_1(x)}}\le(1-\sqrt{1-\eps})^2\le\big(1-(1-\eps)\big)^2=\eps^2$. If $D_2(x)\ge D_1(x)$, we have $1-\sqrt{\frac{D_2(x)}{D_1(x)}}\le(\sqrt{1+\eps}-1)^2\le\big((1+\eps-1)\big)^2=\eps^2$. Combining these two cases, we have
    \[
    d^2_{\textsc{H}}(D_1,D_2)\le\frac{1}{2}\sum_{x\in\+X}D_1(x)\eps^2=\frac{1}{2}\eps^2.
    \]
\end{proof}

Finally, let $D^{\bigotimes m}$ denote the empirical distribution of $T$ i.i.d samples from $D$, namely, the product of $m$ independent $D$ distributions. We have the following lemma relates $d^2_{\textsc{H}}(D_1^{\bigotimes m}, D_2^{\bigotimes m})$ with $d^2_{\textsc{H}}(D_1,D_2)$.
\begin{lemma}\label{empirical}(\cite{Jasper2020note11})
    $d^2_{\textsc{H}}(D_1^{\bigotimes m}, D_2^{\bigotimes m})=1-\InParentheses{1-d^2_{\textsc{H}(D_1,D_2)}}^m\le m\cdot d^2_{\textsc{H}}(D_1,D_2)$.
\end{lemma}

\subsection{Distinguishing distributions}
\label{app:distinguishing}
Let $D_1, D_2$ be two distributions over a discrete space $\+X$. A distribution $D_i$ is chosen uniformly from the set $\{D_1,D_2\}$. Then we are given $m$ samples from $D_i$ and want to distinguish whether the distribution is $D_1$ or $D_2$. It is known that at least $m=\Omega\InParentheses{\frac{1}{d^2_{\textsc{H}}(D_1,D_2)}\log\frac{1}{\delta}}$ samples are needed to guess correctly with probability at least $1-\delta$, no matter how we guess. Formally we have
\begin{lemma}\label{two}
    Let $j\in\{1,2\}$ be the index of the distribution we guess based on the samples. The probability of making a mistake when distinguishing $D_1$ and $D_2$ using $m$ samples, namely $\Pr[j\neq i]=\frac{1}{2}\Pr[j=2|i=1]+\frac{1}{2}\Pr[j=1|i=2]$, is at least
        \[
        \Pr[j\neq i]\ge\frac{1}{4}\InParentheses{1-d^2_{\textsc{H}}(D_1,D_2)}^{2m}\ge\frac{1}{4}e^{-4md^2_{\textsc{H}}(D_1,D_2)},
        \]
    if $d^2_{\textsc{H}}(D_1,D_2)\le\frac{1}{2}$. The inequality implies that, in order to achieve $\Pr[j\neq i]\le\delta$, we must have $m\ge\frac{1}{4d^2_{\textsc{H}}(D_1,D_2)}\log\frac{1}{4\delta}$.
\end{lemma}
\begin{proof}
    The draw of $m$ samples from $D_1$ or $D_2$ is equivalent to the draw of one sample from $D_1^{\bigotimes m}$ or $D_2^{\bigotimes m}$. Given one sample from $D_1^{\bigotimes m}$ or $D_2^{\bigotimes m}$, the probability of making one mistake when guessing the distribution is at least
    \begin{equation}\label{distinguish}
    \begin{aligned}
        \Pr[j\neq i]&=\frac{1}{2}\Pr[j=2|i=1]+\frac{1}{2}\Pr[j=1|i=2]\\
                    &=\frac{1}{2}\InParentheses{1-\Pr[j=1|i=1]+\frac{1}{2}\Pr[j=1|i=2]}\\
                    &=\frac{1}{2}-\frac{1}{2}\InParentheses{\Pr[j=1|i=2]-\Pr[j=1]|i=1]}\\
                    &=\frac{1}{2}-\frac{1}{2}\InParentheses{\E_{D_1^{\bigotimes m}}[\1_{j=1}]-\E_{D_2^{\bigotimes m}}[\1_{j=1}]}\\
                    by \cref{Tfact}&\ge \frac{1}{2}-\frac{1}{2}d_{\textsc{TV}}\InParentheses{D_1^{\bigotimes m},D_2^{\bigotimes m}}.
    \end{aligned} 
    \end{equation}

    Then we upper bound $d_{\textsc{TV}}\InParentheses{D_1^{\bigotimes m}, D_2^{\bigotimes m}}$ to prove the lemma. According to \cref{HFact} and \cref{empirical}, we have
    \[
    1-d^2_{\textsc{TV}}\InParentheses{D_1^{\bigotimes m},D_2^{\bigotimes m}}\ge\InParentheses{1-d^2_{\textsc{H}}\InParentheses{D_1^{\bigotimes m},D_2^{\bigotimes m}}^2}=\InParentheses{1-d^2_{\textsc{H}}\InParentheses{D_1,D_2}}^{2m}.
    \]

    Since \[
    \begin{aligned}
           1-d^2_{\textsc{TV}}\InParentheses{D_1^{\bigotimes m},D_2^{\bigotimes m}}&=\InParentheses{1+d_{\textsc{TV}}\InParentheses{D_1^{\bigotimes m},D_2^{\bigotimes m}}}\InParentheses{1-d_{\textsc{TV}}\InParentheses{D_1^{\bigotimes m},D_2^{\bigotimes m}}}\\
           &\le2\InParentheses{1-d_{\textsc{TV}}\InParentheses{D_1^{\bigotimes m},D_2^{\bigotimes m}}}, \\
    \end{aligned}
    \]we have
    \[
        1-d_{\textsc{TV}}\InParentheses{D_1^{\bigotimes m},D_2^{\bigotimes m}}\ge\frac{1}{2}\InParentheses{1-d^2_{\textsc{H}}\InParentheses{D_1,D_2}}^{2m}.
    \]
    Plugging into \cref{distinguish}, we have
    \[
    \Pr[j\neq i]\ge\frac{1}{4}\InParentheses{1-d^2_{\textsc{H}}(D_1,D_2)}^{2m}.
    \]
    When $d^2_{\textsc{H}}(D_1,D_2)<\frac{1}{2}$, the inequality $1-x\ge e^{-2x}$ for all $x\in(0,\frac{1}{2})$ concludes that
    \[
    \Pr[j\neq i]\ge\frac{1}{4}e^{-4md^2_{\textsc{H}}(D_1,D_2)}.
    \]
 \end{proof}

\section{Missing Proofs from \texorpdfstring{\cref{Pos}}{Section 4}}
\subsection{Proof of \texorpdfstring{\cref{Liplemma}}{Lemma 4.2}}
\begin{proof}
    For any $0 \le \gamma_1 \le \gamma_2 \le 1$,
    \begin{equation*}
    \begin{aligned}
        U(\gamma_2) - U(\gamma_1) &= \E_{v\sim F}\InBrackets{g(\gamma_2,v) \cdot \1_{v\ge\gamma_2}} - \E_{v\sim F}[g(\gamma_1,v) \cdot \1_{v\ge\gamma_1}]\\
        &=\E_{v\sim F}[\big(g(\gamma_2,v) - g(\gamma_1,v)\big)\cdot\1_{v\ge\gamma_2}] - \E_{v\sim F}[g(\gamma_1,v)\cdot\1_{\gamma_1\le v < \gamma_2}]\\
        &\le\E_{v\sim F}[\big(g(\gamma_2,v) - g(\gamma_1,v)\big)\cdot\1_{v\ge\gamma_2}] ~~~~~ \le ~~~~ L(\gamma_2-\gamma_1)
    \end{aligned}
    \end{equation*}
    where the first inequality holds because $g(\gamma,v)\ge0$, the second inequality holds because the projection $g_v$ is right-Lipschitz continuous and the expectation $\E_{v\sim F}[\1_{v\ge\gamma_2}] \le 1$. 
\end{proof}

\subsection{Proof of \texorpdfstring{\cref{Lipupper}}{Theorem 4.2}}
\begin{proof}
 Because $\mathcal{C}_{\textsc{LIP}}\subset\mathcal{C}_{\textsc{ALL}}$, we have $\QUERYCPLX_{\mathcal{G}_{\textsc{LIP}},\mathcal{C}_{\textsc{LIP}}}(\eps,\delta) \le \QUERYCPLX_{\mathcal{G}_{\textsc{LIP}},\mathcal{C}_{\textsc{ALL}}}(\eps,\delta)$ for all $\eps>0$, $\delta\in(0,1)$.

    The expected utility $U$ is right-sided Lipschitz continuous according to \cref{Liplemma}. We can still consider all multiples of $\frac{\eps}{L}$ in $[0,1]$ and define $\hat{U}$ similarly.

    Let $\gamma^*\in\argmax_{\gamma\in[0,1]}U(\gamma)$ be an optimal threshold. And let $\hat{\gamma}^*\in\argmax_{\gamma\in\Gamma}$ be the optimal threshold on the discretized set.
And let $\gamma_l=\frac{\eps}{L}\lfloor\frac{L\gamma^*}{\eps}\rfloor$ respectively be the multiples of $\frac{\eps}{L}$ closest to the left. Then we have $\gamma_l\in\Gamma$ and $0<\gamma^*-\gamma_l<\eps$. Since the reward function g is right-$L$-Lipschitz-continuous, 
\begin{equation*}
\begin{aligned}
        U(\hat{\gamma}^*)&\ge\hat{U}(\hat{\gamma}^*)-\eps\ge\hat{U}(\gamma_l)-\eps\\
        &\ge U(\gamma_l)-2\eps\\
        &\ge U(\gamma^*)-2\eps-L(\gamma^*-\gamma_l)\\
        &\ge U(\gamma^*)-3\eps
\end{aligned}
\end{equation*}
where the first and third inequality holds because \cref{single-est}, the second inequality holds because the selection of $\hat{\gamma}^*$, the fourth inequality holds because of \cref{Liplemma}.    
\end{proof}
    
\subsection{Proof of \texorpdfstring{\cref{claim:F_w_eps}}{Claim 4.1}}
\label{app:proof:claim:F_w_eps}
    \begin{proof}
    For $v<w-3\eps$ and $v>w+3\eps$, $\int_{t=0}^vh_{w,\eps}(t)dt=0$ because $h_{w,\eps}(t)=0$ when $t<w-3\eps$ or $t>w+3\eps$ and $\int_{t=w-3\eps}^{w+3\eps} h_{w,\eps}(t)dt=3\eps-3\eps=0$. Therefore, for $v<w-3\eps$ and $v>w+3\eps$, $F_{w,\eps}(v)=F_0(v)$.
    For any $v\in[w-3\eps,w)$, 
    \begin{equation*}
        F_{w,\eps}(v)-F_0(v)=\int_{t=w-3\eps}^vh_{w,\eps}(t)dt=-(v-w+3\eps). 
    \end{equation*}
    And for any $v\in[w,w+3\eps]$,    \begin{equation*}
    \begin{aligned}
        F_{w,\eps}(v)-F_0(v)=\int_{w-3\eps}^vh_{w,\eps}(t)dt =\int_{w-3\eps}^{w+3\eps}h_{w,\eps}(t)dt-\int_{v}^{w+3\eps}h_{w,\eps}(t)dt=-(w+3\eps-v). 
    \end{aligned}
    \end{equation*}
    \end{proof}

\subsection{Proof of \texorpdfstring{\cref{dis}}{Lemma 4.3}}\label{app:dis}
\begin{proof}
When the learner sets different thresholds $\gamma$ and the value distribution is $F_0$, the samples come from different distributions $G_{\gamma}$. Similarly, when the learner sets different thresholds $\gamma$ and the value distribution is $F_{w,\eps}$, assume that the samples come from different distributions $G^{w,\eps}_{\gamma}$.

In order to distinguish $F_0$ and $F_{w,\eps}$, the learner must at least find a threshold $\gamma$ that it is able to distinguish $G_{\gamma}$ and $G^{w,\eps}_{\gamma}$.

Given $g(\gamma,v)=\gamma\cdot\1_{v\ge\gamma}$, $G_{\gamma}$ is a Bernoulli distribution that for all $X\sim G_{\gamma}$, $\Pr[X=0]=F_0(\gamma)$ and $\Pr[X=\gamma]=1-F_0(\gamma)$. And $G_{\gamma}^{w,\eps}$ is also a Bernoulli distribution that for all $Y\sim G_{\gamma}^{w,\eps}$, $\Pr[Y=0]=F_{w,\eps}(\gamma)$ and $\Pr[Y=\gamma]=1-F_{w,\eps}(\gamma)$.

Recall that \begin{equation*}
    F_{w,\eps}(v)=
    \begin{cases}
        F_0(v)-(v-w+3\eps) & v\in[w-3\eps,w)\\
        F_0(v)-(w+3\eps-v) & v\in[w,w+3\eps]\\
        F_0(v) & \text{otherwise}
    \end{cases}
\end{equation*}

When $\gamma<w-3\eps$ or $\gamma>w+3\eps$, $G_{\gamma}$ and $G^{w,\eps}_{\gamma}$ are the same Bernoulli distribution. The learner can't distinguish them. When $\gamma\in[w-3\eps,w+3\eps]$, we have $F_0(\gamma)\ge F_{w,\eps}(\gamma)\ge F_0(\gamma)-3\eps$ and $F_0(\frac{1}{2})\ge F_0(\gamma)\ge F_0(\frac{1}{3})$. Recall $F_0(\frac{1}{3})=\frac{1}{4}$ and $F_0(\frac{1}{2})=\frac{1}{2}$ .Then we get
\[
1\ge\frac{\Pr[Y=0]}{\Pr[X=0]}=\frac{F_{w,\eps}(\gamma)}{F_0(\gamma)}\ge\frac{F_0(\gamma)-3\eps}{F_0(\gamma)}=1-\frac{3\eps}{F_0(\gamma)}\ge 1-\frac{3\eps}{F_0(\frac{1}{3})}=1-12\eps.
\]

and 

\[
1\le\frac{\Pr[Y=\gamma]}{\Pr[X=\gamma]}=\frac{1-F_{w,\eps}(\gamma)}{1-F_0(\gamma)}\le\frac{1-F_0(\gamma)+3\eps}{1-F_0(\gamma)}=1+\frac{3\eps}{1-F_0(\gamma)}\le1+\frac{3\eps}{1-F_0(\frac{1}{2})}=1+6\eps.
\]

According to \cref{Hupperbound}, we have $d^2_{\textsc{H}}\InParentheses{G_{\gamma},G^{w,\eps}_{\gamma}}\le\frac{1}{2}(12\eps)^2\le72\eps^2$. Then we know from \cref{two} that $\Omega\InParentheses{\frac{1}{\eps^2}\log\frac{1}{\delta}}$ samples are needed to distinguish $G_{\gamma}$ and $G^{w,\eps}_{\gamma}$ with probability at least $1-\delta$.

In other words, the learner must at least find a threshold $\gamma\in[w-3\eps,w+3\eps]$ and do at least $\Omega\InParentheses{\frac{1}{\eps^2}\log\frac{1}{\delta}}$ queries at the same threshold $\gamma$ to distinguish $F_0$ and $F_{w,\eps}$ with probability at least $1-\delta$.
\end{proof}

\section{Proof of \texorpdfstring{\cref{noregret}}{Theorem 5.1}}
\label{app:proof:noregret} 
\begin{proof}
    First, because $v_t\sim F_t$ is independent of $\gamma_t$, we have $\E_{v_t\sim F_t}[b_t(\gamma_t, v_t)] = U_t(\gamma_t)$, and we can rewrite the regret as 
    \begin{equation*}
        \sup_{\gamma\in[0,1]}\-E\InBrackets{\sum_{t=1}^Tb_t(\gamma,v_t)-\sum_{t=1}^Tb_t(\gamma_t,v_t)} = \sup_{\gamma\in[0,1]} \bigg\{ \-E\Big[\sum_{t=1}^T U_t(\gamma)\Big] - \-E\Big[\sum_{t=1}^T U_t(\gamma_t)\Big] \bigg\}, 
    \end{equation*}
    where the expectation on the right-hand-side is only over the randomness of algorithm $\+A$ but not $v_t$. 

    We treat the online learning problem as a continuous-arm adversarial bandit problem, where each threshold $\gamma \in [0, 1]$ is an arm.
    According to \Cref{mono} and \Cref{Liplemma}, in all the three environments $\+S_1, \+S_2, \+S_3$ in the theorem the expected utility function $U_t(\gamma) = \E_{v_t\sim F_t}[b_t(\gamma, v_t)]$ is one-sided Lipschitz in $\gamma$.
    W.l.o.g, assume that $U_t$ is right-Lipschitz. Let's discretize the arm space $[0,1]$ uniformly with interval length $\eps$, obtaining a finite set of arms $\Gamma=\{0,\eps,2\eps,...\}$ with $\InAbs{\Gamma}\le\frac{1}{\eps}+1$. Let $\gamma^*\in\argmax_{\gamma\in[0,1]} \-E\big[ \sum_{t=1}^T U_t(\gamma) \big]$ be an optimal threshold in the interval $[0, 1]$ (for the expected sum of utility functions).  And let $\hat{\gamma}^*\in\argmax_{\gamma\in\Gamma} \-E\big[ \sum_{t=1}^T U_t(\gamma) \big]$ be an optimal threshold in the discretized set $\Gamma$.
    And let $\hat{\gamma}_l \in \Gamma$ be the largest multiple of $\eps$ that does not exceed $\gamma^*$.  Clearly, $\gamma^* - \hat{\gamma}_l \le \eps$.  Because every $U_t$ is right-Lipschitz, we have
    \begin{align*}
        \-E\Big[ \sum_{t=1}^T U_t(\gamma^*) \Big] - \-E\Big[ \sum_{t=1}^T U_t(\hat \gamma_l) \Big] \le \sum_{t=1}^T L(\gamma^* - \hat{\gamma}_l) \le TL \eps. 
    \end{align*}
    This implies that the optimal threshold $\hat{\gamma}^*$ in $\Gamma$ satisfies
    \begin{align*}
        \-E\Big[\sum_{t=1}^T U_t(\hat \gamma^*)\Big] \ge \-E\Big[ \sum_{t=1}^T U_t(\hat \gamma_l) \Big] \ge \-E\Big[ \sum_{t=1}^T U_t(\gamma^*) \Big] - TL \eps. 
    \end{align*}
    Recall that the Poly INF algorithm (Theorem 11 of \cite{audibert_regret_2010}) is an adversarial multi-armed bandit algorithm with $O(\sqrt{TK})$ regret when running on an arm set of size $K$.  If we run that algorithm on the arm set $\Gamma$, and let $\eps = T^{-1/3}$, then we get a total expected utility of at least
    \begin{align*}
        \-E\Big[ \sum_{t=1}^T U_t(\gamma_t) \Big] & \ge \-E\Big[ \sum_{t=1}^T U_t(\hat\gamma^*) \Big] - O(\sqrt{T|\Gamma|}) \\
        & \ge \-E\Big[ \sum_{t=1}^T U_t(\gamma^*) \Big] - TL \eps - O(\sqrt{T\tfrac{1}{\eps}}) \\
        & = \-E\Big[ \sum_{t=1}^T U_t(\gamma^*) \Big] - O(T^{2/3}). 
    \end{align*}
    So, the regret is at most $O(T^{2/3})$. 
\end{proof}

\section{Discussions on the Lipschitz constant L}
\label{app:L}
The proof of \cref{MONO+LIP} relies on knowing the Lipschitz constant $L$ and the sample complexity upper bound is $O\InParentheses{\frac{L}{\eps^3}\log\frac{1}{\eps\delta}}$. In this section, we adapt the proof of \cref{MONO+LIP} to cope with the case where the Lipschitz constant $L$ is unknown. We prove a new sample complexity upper bound $\bm {O\InParentheses{\frac{1}{\eps^3}\log\frac{L}{\eps}\log\frac{\log\frac{L}{\eps}}{\eps\delta}}}$, which improves the $\frac{L}{\eps^3}$ term to $\frac{1}{\eps^3}$ while has an additional $\log L$ terms compared with the case that we know $L$.

The function of set $\Gamma$ is to discretize on $[0,1]$ and ensure that the expected reward gap between any two adjacent points is at most $\eps$. However, if we scrutinize the proof of \cref{mono}, we have $U(\gamma_2)-U(\gamma_1)\ge -(F(\gamma_2)-F(\gamma_1))$ for any $0\le\gamma_1\le\gamma_2\le1$. Therefore, it is sufficient to discretize on $[0,1]$ and ensure that the CDF gap between any two adjacent points is at most $\eps$. Next, we show how to adaptively build a discretization set $\Gamma_A$ holding the above property without knowing $L$.

By Chernoff bound, we know $O\InParentheses{\frac{1}{\eps^2}\log\frac{1}{\delta}}$ queries are sufficient to learn $F(x)$ with $\frac{\eps}{9}$ additive error for any $x\in[0,1]$. At step 1, let $\Gamma_A=\{0\}$. At step $n>1$, assuming $\Gamma_A=\{x_1,x_2,...,x_{n-1}\}$and the estimations of corresponding CDF $\hat{F}(x_1),\hat{F}(x_2),...,\hat{F}(x_{n-1})$ are computed. Let $x_{n-1}=\max{\Gamma_A}$. Then we use binary search to find the next element $x_n$ satisfying $\frac{2\eps}{3}-\frac{\eps}{9}<\hat{F}(x_n)-\hat{F}(x_{n-1})<\frac{2\eps}{3}+\frac{\eps}{9}$ and hence $\frac{\eps}{3}<F(x_n)-F(x_{n-1})<\eps$. The binary search works because of the intrinsic monotonicity of empirical CDF function $\hat{F}$. Note that we always can find such $x_n$ within $O\InParentheses{\log \frac{L}{\eps}}$ points because of the Lipschitzness. Overall, to build $\Gamma_A$, we need to estimate $n=O\InParentheses{\log\frac{L}{\eps}\cdot\InAbs{\Gamma_A}}=O\InParentheses{\frac{1}{\eps}\log\frac{L}{\eps}}$ points of the value distribution. By union bound, to successfully build $\Gamma_A$ with probability $1-\frac{\delta}{2}$, the total query complexity is $O\InParentheses{n\cdot\frac{1}{\eps^2}\log\frac{n}{\delta}}=O\InParentheses{\frac{1}{\eps^3}\log\frac{L}{\eps}\log\frac{\log\frac{L}{\eps}}{\eps\delta}}$. Once $\Gamma_A$ is built, following the proof of \cref{MONO+LIP}, we have $O\InParentheses{\frac{1}{\eps^3}\log\frac{1}{\eps\delta}}$ queries are sufficient to learn the optimal threshold within $\eps$ additive error with probability at least $1-\frac{\delta}{2}$. Combining the above, $O\InParentheses{\frac{1}{\eps^3}\log\frac{L}{\eps}\log\frac{\log\frac{L}{\eps}}{\eps\delta}}$ queries are sufficient to build a $(\eps,\delta)$-estimator.

\section{Experimental results}
In this section, we provide some simple experiments to verify our theoretical results.
\subsection{Upper bound}
We consider two toy examples. The first example is $g(\gamma,v)=\gamma$ if $\gamma<\frac{1}{3}$ otherwise $g(\gamma,v)=v$ and the value distribution is the uniform distribution on $[0,1]$, which corresponds to the monotone reward function and Lipschitz value distribution. The second example is $g(\gamma,v)=\gamma$ and the value distribution is a point distribution that all the mass is on $v=\frac{1}{3}$, which corresponds to the Lipschitz reward function and general distribution case (\cref{Lipupper}). In our experiment, we first fix the number of queries to be $K^3$ where $K=100+3i$ for all integer $1\le i\le 33$, i.e. we choose $K$ from $[100,200]$. Therefore, the algorithm for the upper bound should output errors smaller than the predetermined loss $\frac{1}{K}$. The following figure shows the relationship between the loss and the number of queries. The empirical loss curve is under the predetermined loss curve, which verifies our upper bound results.
\begin{figure}[!h]
    \centering
    \includegraphics[width=65mm]{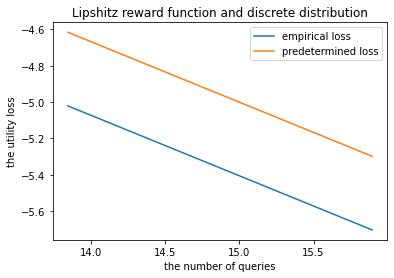}
    \includegraphics[width=65mm]{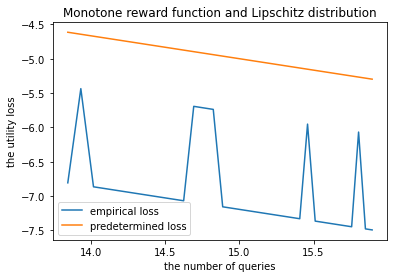}
    \caption{Loss curves under different examples. The orange line: the predetermined loss curve. The blue line: the empirical loss curve. All variables are in logarithmic form.}
    \label{fig:upper-bound experiment}
\end{figure}

\subsection{lower bound}
In this section, we provide experimental results to verify our lower bound result (\cref{lower}). We consider the example we provided in the proof of \cref{lower} where $g(\gamma,v)=\gamma$ and the value distribution is a ``hard distribution'' (see the left part of \cref{fig:lower-bound experiment}). For $\eps\in\{\frac{1}{400},\frac{1}{500},\frac{1}{600}\}$, we run the algorithm in \cref{MONO+LIP} to determine the minimum number $n$ of queries that are necessary to learn the optimal threshold with $\eps$ additive error. Due to randomness, we repeat 10 times for each $\eps$. At each round, we compute $\frac{\ln n}{\ln \frac{1}{\eps}}$ and find it converging to $3$, which verifies our lower bound result. 

\begin{figure}[t]
    \centering
    \includegraphics[width=65mm]{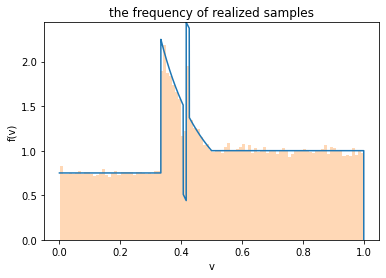}
    \includegraphics[width=65mm]{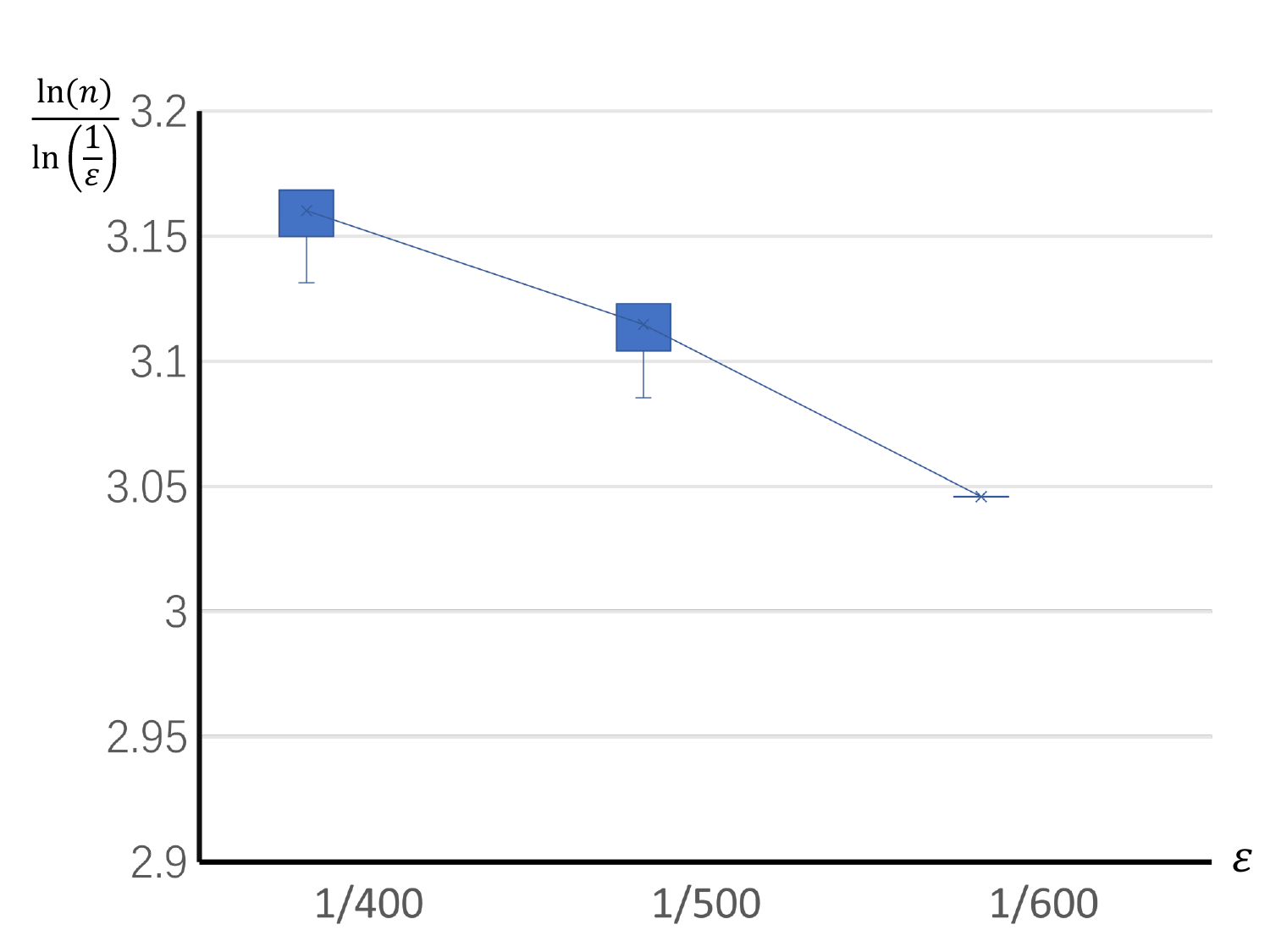}
    \caption{Left: The blue curve is the probability distribution function. The orange part is the frequency of realized samples. Right: The box plot when $\eps\in\{\frac{1}{400},\frac{1}{500},\frac{1}{600}\}$. The horizontal axe represents $\eps$. The vertical axe represents the logarithmic ratio $\frac{\ln n}{\ln \frac{1}{\eps}}$.}
    \label{fig:lower-bound experiment}
\end{figure}

\end{document}

%% file: math_commands.tex
\usepackage{amsmath,amsfonts,bm}

\def\eqref#1{equation~\ref{#1}}

\def\1{\bm{1}}

\def\eps{{\varepsilon}}

\DeclareMathAlphabet{\mathsfit}{\encodingdefault}{\sfdefault}{m}{sl}
\SetMathAlphabet{\mathsfit}{bold}{\encodingdefault}{\sfdefault}{bx}{n}

\def\sN{{\mathbb{N}}}

\def\sP{{\mathbb{P}}}

\def\sR{{\mathbb{R}}}

\newcommand{\E}{\mathbb{E}}

\DeclareMathOperator*{\argmax}{arg\,max}

%% file: main.bbl
\begin{thebibliography}{10}

\bibitem{abernethy_threshold_2016}
Jacob~D Abernethy, Kareem Amin, and Ruihao Zhu.
\newblock Threshold {Bandits}, {With} and {Without} {Censored} {Feedback}.
\newblock In {\em Advances in {Neural} {Information} {Processing} {Systems}}, volume~29. Curran Associates, Inc., 2016.

\bibitem{agrawal1995continuum}
Rajeev Agrawal.
\newblock The continuum-armed bandit problem.
\newblock {\em SIAM journal on control and optimization}, 33(6):1926--1951, 1995.

\bibitem{amemiya1973regression}
Takeshi Amemiya.
\newblock Regression analysis when the dependent variable is truncated normal.
\newblock {\em Econometrica: Journal of the Econometric Society}, pages 997--1016, 1973.

\bibitem{audibert_regret_2010}
Jean-Yves Audibert and Sébastien Bubeck.
\newblock Regret {Bounds} and {Minimax} {Policies} under {Partial} {Monitoring}.
\newblock {\em Journal of Machine Learning Research}, 11(94):2785--2836, 2010.

\bibitem{auer1995gambling}
Peter Auer, Nicolo Cesa-Bianchi, Yoav Freund, and Robert~E Schapire.
\newblock Gambling in a rigged casino: The adversarial multi-armed bandit problem.
\newblock In {\em Proceedings of IEEE 36th annual foundations of computer science}, pages 322--331. IEEE, 1995.

\bibitem{breen1996regression}
Richard Breen.
\newblock {\em Regression models: Censored, sample selected, or truncated data}, volume 111.
\newblock Sage, 1996.

\bibitem{brustle_multi-item_2020}
Johannes Brustle, Yang Cai, and Constantinos Daskalakis.
\newblock Multi-{Item} {Mechanisms} without {Item}-{Independence}: {Learnability} via {Robustness}.
\newblock In {\em Proceedings of the 21st {ACM} {Conference} on {Economics} and {Computation}}, pages 715--761, Virtual Event Hungary, July 2020. ACM.

\bibitem{cesa2023bilateral}
Nicol{\`o} Cesa-Bianchi, Tommaso Cesari, Roberto Colomboni, Federico Fusco, and Stefano Leonardi.
\newblock Bilateral trade: A regret minimization perspective.
\newblock {\em Mathematics of Operations Research}, 2023.

\bibitem{cesa2023role}
Nicol{\`o} Cesa-Bianchi, Tommaso Cesari, Roberto Colomboni, Federico Fusco, and Stefano Leonardi.
\newblock The role of transparency in repeated first-price auctions with unknown valuations.
\newblock {\em arXiv preprint arXiv:2307.09478}, 2023.

\bibitem{cesa2021regret}
Nicol{\`o} Cesa-Bianchi, Tommaso~R Cesari, Roberto Colomboni, Federico Fusco, and Stefano Leonardi.
\newblock A regret analysis of bilateral trade.
\newblock In {\em Proceedings of the 22nd ACM Conference on Economics and Computation}, pages 289--309, 2021.

\bibitem{cesa2014regret}
Nicolo Cesa-Bianchi, Claudio Gentile, and Yishay Mansour.
\newblock Regret minimization for reserve prices in second-price auctions.
\newblock {\em IEEE Transactions on Information Theory}, 61(1):549--564, 2014.

\bibitem{cole2014sample}
Richard Cole and Tim Roughgarden.
\newblock The sample complexity of revenue maximization.
\newblock In {\em Proceedings of the forty-sixth annual ACM symposium on Theory of computing}, pages 243--252, 2014.

\bibitem{daskalakis2018efficient}
Constantinos Daskalakis, Themis Gouleakis, Chistos Tzamos, and Manolis Zampetakis.
\newblock Efficient statistics, in high dimensions, from truncated samples.
\newblock In {\em 2018 IEEE 59th Annual Symposium on Foundations of Computer Science (FOCS)}, pages 639--649. IEEE, 2018.

\bibitem{pmlr-v99-daskalakis19a}
Constantinos Daskalakis, Themis Gouleakis, Christos Tzamos, and Manolis Zampetakis.
\newblock Computationally and statistically efficient truncated regression.
\newblock In {\em Proceedings of the Thirty-Second Conference on Learning Theory}, pages 955--960, 2019.

\bibitem{duetting2023optimal}
Paul Duetting, Guru Guruganesh, Jon Schneider, and Joshua~Ruizhi Wang.
\newblock Optimal no-regret learning for one-sided lipschitz functions.
\newblock In {\em International Conference on Machine Learning}, 2023.

\bibitem{dvoretzky_asymptotic_1956}
A.~Dvoretzky, J.~Kiefer, and J.~Wolfowitz.
\newblock Asymptotic {Minimax} {Character} of the {Sample} {Distribution} {Function} and of the {Classical} {Multinomial} {Estimator}.
\newblock {\em The Annals of Mathematical Statistics}, 27(3):642--669, September 1956.

\bibitem{eldowa2023minimax}
Khaled Eldowa, Emmanuel Esposito, Tommaso Cesari, and Nicol{\`o} Cesa-Bianchi.
\newblock On the minimax regret for online learning with feedback graphs.
\newblock {\em arXiv preprint arXiv:2305.15383}, 2023.

\bibitem{feng2021reserve}
Zhe Feng, S{\'e}bastien Lahaie, Jon Schneider, and Jinchao Ye.
\newblock Reserve price optimization for first price auctions in display advertising.
\newblock In {\em International Conference on Machine Learning}, pages 3230--3239. PMLR, 2021.

\bibitem{gonczarowski2017efficient}
Yannai~A Gonczarowski and Noam Nisan.
\newblock Efficient empirical revenue maximization in single-parameter auction environments.
\newblock In {\em Proceedings of the 49th Annual ACM SIGACT Symposium on Theory of Computing}, pages 856--868, 2017.

\bibitem{guinet_effective_2022}
Gauthier Guinet, Saurabh Amin, and Patrick Jaillet.
\newblock Effective {Dimension} in {Bandit} {Problems} under {Censorship}.
\newblock In {\em Advances in {Neural} {Information} {Processing} {Systems}}, volume~35, pages 5243--5255. Curran Associates, Inc., 2022.

\bibitem{guo2019settling}
Chenghao Guo, Zhiyi Huang, and Xinzhi Zhang.
\newblock Settling the sample complexity of single-parameter revenue maximization.
\newblock In {\em Proceedings of the 51st Annual ACM SIGACT Symposium on Theory of Computing}, pages 662--673, 2019.

\bibitem{haghtalab2022smoothed}
Nika Haghtalab, Tim Roughgarden, and Abhishek Shetty.
\newblock Smoothed analysis with adaptive adversaries.
\newblock In {\em 2021 IEEE 62nd Annual Symposium on Foundations of Computer Science (FOCS)}, pages 942--953. IEEE, 2022.

\bibitem{kleinberg2003value}
Robert Kleinberg and Tom Leighton.
\newblock The value of knowing a demand curve: Bounds on regret for online posted-price auctions.
\newblock In {\em 44th Annual IEEE Symposium on Foundations of Computer Science, 2003. Proceedings.}, pages 594--605. IEEE, 2003.

\bibitem{kleinberg2019bandits}
Robert Kleinberg, Aleksandrs Slivkins, and Eli Upfal.
\newblock Bandits and experts in metric spaces.
\newblock {\em Journal of the ACM (JACM)}, 66(4):1--77, 2019.

\bibitem{kontonis2019efficient}
Vasilis Kontonis, Christos Tzamos, and Manolis Zampetakis.
\newblock Efficient truncated statistics with unknown truncation.
\newblock In {\em 2019 IEEE 60th Annual Symposium on Foundations of Computer Science (FOCS)}, pages 1578--1595. IEEE, 2019.

\bibitem{auction-theory-book}
Vijay Krishna.
\newblock In {\em Auction Theory (Second Edition)}, page iii. Academic Press, San Diego, second edition edition, 2010.

\bibitem{Jasper2020note11}
Jasper Lee.
\newblock Lecture 11: Distinguishing (discrete) distributions.
\newblock 2020.

\bibitem{leme2023description}
Renato~Paes Leme, Balasubramanian Sivan, Yifeng Teng, and Pratik Worah.
\newblock Description complexity of regular distributions.
\newblock {\em arXiv preprint arXiv:2305.05590}, 2023.

\bibitem{leme2023pricing}
Renato~Paes Leme, Balasubramanian Sivan, Yifeng Teng, and Pratik Worah.
\newblock Pricing query complexity of revenue maximization.
\newblock In {\em Proceedings of the 2023 Annual ACM-SIAM Symposium on Discrete Algorithms (SODA)}, pages 399--415. SIAM, 2023.

\bibitem{little_statistical_2020}
Roderick J.~A. Little and Donald~B. Rubin.
\newblock {\em Statistical {Analysis} with {Missing} {Data}}.
\newblock Wiley series in probability and statistics. Wiley, Hoboken, NJ, third edition edition, 2020.

\bibitem{magureanu2014lipschitz}
Stefan Magureanu, Richard Combes, and Alexandre Proutiere.
\newblock Lipschitz bandits: Regret lower bound and optimal algorithms.
\newblock In {\em Conference on Learning Theory}, pages 975--999. PMLR, 2014.

\bibitem{massart_tight_1990}
P.~Massart.
\newblock The {Tight} {Constant} in the {Dvoretzky}-{Kiefer}-{Wolfowitz} {Inequality}.
\newblock {\em The Annals of Probability}, 18(3), July 1990.

\bibitem{morgenstern2015pseudo}
Jamie~H Morgenstern and Tim Roughgarden.
\newblock On the pseudo-dimension of nearly optimal auctions.
\newblock {\em Advances in Neural Information Processing Systems}, 28, 2015.

\bibitem{tobin1958estimation}
James Tobin.
\newblock Estimation of relationships for limited dependent variables.
\newblock {\em Econometrica: journal of the Econometric Society}, pages 24--36, 1958.

\bibitem{verma_censored_2019}
Arun Verma, Manjesh Hanawal, Arun Rajkumar, and Raman Sankaran.
\newblock Censored {Semi}-{Bandits}: {A} {Framework} for {Resource} {Allocation} with {Censored} {Feedback}.
\newblock In {\em Advances in {Neural} {Information} {Processing} {Systems}}, volume~32. Curran Associates, Inc., 2019.

\end{thebibliography}
